\documentclass[notitlepage,a4paper]{scrartcl}

\usepackage[dvipsnames]{xcolor}
\definecolor{light-gray}{gray}{0.8}
\definecolor{light-red}{rgb}{1,0.5,0.5}

\usepackage{my/setup}
\usepackage{my/symbols}
\usepackage[decisionutilitycolor, chamferedutility]{my/influence-diagrams}
\usepackage{my/cite}
\bibliography{library-manual.bib}
\bibliography{library.bib}

\usepackage{bm}

\newcommand*{\dom}{\textit{dom}}

\newcommand*{\sW}{{\bm W}}

\newcommand*{\sA}{{\bm D}}

\newcommand*{\sS}{{\bm S}}
\renewcommand*{\ss}{{\bm s}}
\newcommand*{\sX}{{\bm X}}

\newcommand*{\sZ}{{\bm Z}}

\newcommand*{\sO}{{\bm O}}
\newcommand*{\so}{{\bm o}}

\newcommand*{\sU}{{\bm U}}

\newcommand*{\cg}[1][]{(\sW#1, E #1)}
\newcommand*{\cm}[1][]{(\sW#1, E#1, P#1)}
\newcommand*{\envg}[1][]{(\sW#1, E#1, \sA#1, \sU#1)}
\newcommand*{\envm}[1][]{(\sW#1, E#1, \sA#1, \sU#1, P#1)}
\newcommand*{\envgsa}[1][]{(\sW#1, E#1, \{D\}, \sU#1)}
\newcommand*{\envmsa}[1][]{(\sW#1, E#1, \{D\}, \sU#1, P#1)}

\usepackage{geometry}

\title{
  Understanding Agent Incentives using\\ Causal Influence Diagrams%
  \thanks{
    A number of people have been essential in preparing this paper.
    Ryan Carey,
    Eric Langlois,
    Michael Bowling,
    Tim Genewein,
    James Fox,
    Daniel Filan,
    Ray Jiang,
    Silvia Chiappa,
    Stuart Armstrong,
    Paul Christiano,
    Mayank Daswani,
    Ramana Kumar,
    Jonathan Uesato,
    Adria Garriga,
    Richard Ngo,
    Victoria Krakovna,
    Allan Dafoe,
    and
    Jan Leike
    have all contributed through thoughtful discussions and/or by reading drafts at
    various stages of this project.
  }
}
 \subtitle{Part I: Single Decision Settings}

\author{Tom Everitt\hspace{-5mm} \and Pedro A.\ Ortega\hspace{-5mm} \and
  Elizabeth Barnes \and\hspace{-5mm} Shane Legg}
\publishers{DeepMind}
\date{September 9, 2019}

\begin{document}

\newgeometry{bottom=1.5cm,top=2cm}

\maketitle
\thispagestyle{empty}
\vspace{-5mm}

\begin{abstract}
  Agents are systems that optimize an objective function in an environment.
  Together, the goal and the environment induce secondary objectives, \emph{incentives}.
  Modeling the agent-environment interaction using
  \emph{causal influence diagrams},
  we can answer two fundamental questions about an agent's incentives
  directly from the graph:
  (1) which nodes can the agent have an incentivize to observe, and
  (2) which nodes can the agent have an incentivize to control?
  The answers tell us which information and influence points need extra protection.
  For example, we may want a classifier for job applications to not use the
  ethnicity of the candidate, and a reinforcement learning agent not to take
  direct control of its reward mechanism.
  Different algorithms and training paradigms can lead to different
  causal influence diagrams, so our method can be used to identify algorithms with
  problematic incentives and help in designing algorithms with
  better incentives.
\end{abstract}

\setcounter{tocdepth}{1}
\tableofcontents
\restoregeometry

\section{Introduction}
\label{ch:introduction}

Agents strive to optimize an objective function in an environment.
This gives them \emph{incentives}
to learn about and influence various aspects of the environment.
For example, a reinforcement learning agent
playing the ATARI game Pong will have an incentive to direct the ball to regions
where the opponent will be unable to intercept it,
and have an incentive to learn which those regions are.
The aim of this paper is to provide a simple and systematic method for inferring
agent incentives.
To this end, we define \emph{causal influence diagrams} (CID),
a graphical model with special decision, utility, and chance nodes
\citep{Howard1984},
where all arrows encode causal relationships \citep{Pearl2009}.
CIDs provide a flexible and precise tool for simultaneously describing both
agent objectives and agent-environment interaction.

To determine what information a system wants to obtain in order to
optimize its objective,
we establish a graphical criterion that characterizes which nodes in a CID graph
are compatible with an \emph{observation incentive}.
In words, the criterion is that:
\begin{quote}
  \textbf{Main result 1 (Observation incentives):}
  \emph{A single-decision CID graph is compatible with
    an observation incentive on a node $X$
    if and only if
    $X$ is d-connected to a influenceable utility node
    when conditioning on the decision and all available observations
    (\cref{th:observation-sa}).
  }
\end{quote}
The criterion applies to a conceptually clear definition of observation
incentive, which says that there is an incentive to observe a node if
learning its outcome strictly improves expected utility,
i.e.\ if the node provides a positive \emph{value of information}
\citep{Howard1966}.
Among other things, the criterion detects which observations are useful or
\emph{requisite} when making a decision.
Theorems establishing the \emph{only if} part of observation incentive criterion
have been previously established by
\citet{Fagiuoli1998,Lauritzen2001};
see \cref{sec:related-work} for a more detailed overview.
Here, we also prove the \emph{if} direction.

A related question is what aspects of its environment a system wants to
influence.
To answer this question, we establish an analogous graphical criterion
for \emph{intervention incentives}:
\begin{quote}
  \textbf{Main result 2 (Intervention incentives):}
  \emph{
    A single-decision CID graph is compatible
    with an intervention incentive on a non-decision node $X$ if and only if
    there is a directed path from $X$ to a utility node after all
    nonrequisite information links have been removed (\cref{th:soft-sa}).
}
\end{quote}
Intervention incentives detect a positive
\emph{value of control} \citep{Matheson2005,Shachter2010,Heckerman1995}
or \emph{value of intervention} \citep{Lu2002}.
No graphical criterion of intervention incentives has
previously been established.
Depending on the path from $X$ to the utility node, we can make a further
distinction between whether the intervention on $X$ is used to obtain more
information or to directly control a utility variable.

We demonstrate two applications of our theorems.
The observation incentive criterion provides insights about the \emph{fairness}
of decisions made by machine learning systems and other
agents \citep{ONeil2016},
as it informs us when a variable is likely to be used as a proxy for
a sensitive attribute or not (\cref{sec:fairness-example}).
With the intervention incentive criterion, we study the incentive
of a question-answering system (QA-system)
to influence the world state with its answer,
rather than passively predicting future events
 (\cref{sec:oracle-example}).
Many more applications of CIDs are provided by
\citet{Everitt2019modeling,Everitt2019tampering}.

\paragraph{Outline}
Following an initial background section 
(\cref{sec:background}), we devote
one section to observation incentives (\cref{sec:observation-incentives}) and
one section to intervention incentives (\cref{sec:intervention-incentives}).
These sections contain formal sections defining the criteria,
as well as ``gentler'' sections describing how to use and interpret the criteria.
Both sections conclude with an example application: to fairness for
observation incentives, and to QA-system for intervention incentives.
Finally, we discuss related work (\cref{sec:related-work})
and some open questions (\cref{sec:limitations}),
before stating some conclusions in \cref{sec:conclusions}.
All proofs are deferred to \cref{app:proofs}.

\section{Background}
\label{sec:background}

This section provides the necessary background and notation for the rest of the paper.
A recap of causal graphs (\cref{sec:bayesian-networks}) and d-separation
(\cref{sec:d-separation}) is followed by a definition of CIDs
(\cref{sec:influence-diagrams-sa}).

\subsection{Causal Graphs}
\label{sec:bayesian-networks}

\paragraph{Random Variables}
A random variable is a (measurable) function
$X\colon \Omega \to \dom(X)$ from some measurable space $(\Omega, \Sigma)$ 
to a finite domain $\dom(X)$.
The domain $\dom(X)$ specifies which values the random variable can take.
The outcome of a random variable $X$ is $x$.

A set or vector $\sX = (X_1, \dots, X_n)$ of random variables is again a random
variable, with domain $\dom(\sX) = \prod_{i=1}^n \dom(X_i)$.
We will use boldface font for sets of random variables (e.g.\ $\sX$).

\paragraph{Graphs and models}
Throughout the paper we will make a distinction between \emph{graphs}
on the one hand, and \emph{models} on the other.
A graph only specifies the structure of the interaction,
while a model combines a graph with a \emph{parameterization} to also define
the relationships between the variables.

\begin{definition}[Causal graph; \citealp{Pearl2009}]
  A \emph{causal graph}
  is a directed acyclic graph $\cg$ over a set
  of nodes or random variables $\sW$, connected by edges $E\subseteq \sW\times
  \sW$.
  The arrows indicate the direction of causality, in the sense that
  an external intervention on a node $X$ will affect the descendants
  of $X$, but not the ancestors of $X$.
  We denote the parents of $X$ with $\Pa_X$.
  Following the conventions for random variables, the outcomes of the parent
  nodes are denoted $\pa_X$.
\end{definition}

\begin{figure}
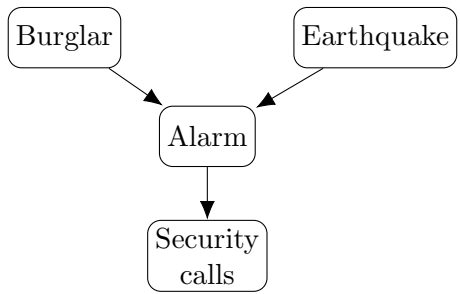

  \centering
  \begin{influence-diagram}
    \node (alarm) {Alarm};
    \node (burglar) [above left = of alarm] {Burglar};
    \node (earth) [above right = of alarm] {Earthquake};
    \node (call) [below = of alarm] {Security\\ calls};

    \path
    (burglar) edge[->] (alarm)
    (earth) edge[->] (alarm)
    (alarm) edge[->] (call)
    ;
  \end{influence-diagram}
  \caption{
    An example of a causal graph \citep{Pearl2009}.
  }
  \label{fig:causal-graph}
\end{figure}
For example, in \cref{fig:causal-graph}, an alarm is influenced
by the presence of a burglar and by a (small) earthquake, and in turn
influences whether the security company calls.
The graph defines the structure of the interaction,
but does not specify the relationship between the variables.
As in a Bayesian network, the precise relationships are specified
by conditional probability distributions $P(x\mid \pa_X)$.

\begin{definition}[Causal model; \citealp{Pearl2009}]
  A \emph{causal model}
  $\cm$ is a causal graph $\cg$ combined with a \emph{parameterization} $P$
  that specifies a finite domain $\dom(X)$ and a
  conditional probability distributions $P(x\mid \pa_X)$
  for each node $X\in \sW$.
\end{definition}

A parameterization $P$ induces a joint distribution
$P(x_1,\dots,x_n) = \prod_{i=1}^n P(x_i\mid \pa_i)$ over all the nodes
$\{X_1,\dots, X_n\} =\sW$.

\subsection{d-Separation}
\label{sec:d-separation}

\begin{definition}[Graph terminology]
  A \emph{path} is a chain of non-repeating nodes connected by edges in the graph. We
  write $X\pathto Y$ for a \emph{directed path} from $X$ to $Y$, and $X\upathto Y$ for
  an \emph{undirected path}. The \emph{length} of a path is the number of edges
  on the path.
  We do allow paths of length 0.
  
  If there is a directed path $X\pathto Y$ of length at least 1,
  then $X$ is an \emph{ancestor} of
  $Y$, and $Y$ is a \emph{descendant} of $X$.
  Let $\desc(X)$ be the set of descendants of $X$.
\end{definition}

An important question is when the outcome of one variable $Y$ provides
information about the outcome of another variable $X$.
This depends, of course, on which other outcomes $\sZ$ that we already know.
If $Y$ provides no additional information about $X$ given that we already
observe $\sZ$, then we say that $X$ and $Y$ are \emph{conditionally independent}
when conditioning on $\sZ$.
Formally, $P(X\mid Y, \sZ) = P(X\mid \sZ)$.
It is possible to tell whether $X$ and $Y$ must be conditionally
independent given $\sZ$ in a causal graph.
The criteria for determining this is called \emph{d-separation}:

\label{sec:d-sep}
\begin{definition}[d-separation; {\citealp{Pearl2009}}]
  \label{def:d-sep}
  \label{def:d-connected}
  An undirected path $X\upathto Y$ in a causal graph is \emph{active}
  conditioning on a set $\sZ$
  if each three node segment of the path subscribes to one of the following
  \emph{active} patterns:
  \begin{itemize}
  \item Chain:  $X_{i-1}\to X_i\to X_{i+1}$ or $X_{i-1}\gets X_i\gets X_{i+1}$ and $X_i\not\in \sZ$.
  \item Fork:  $X_{i-1}\gets X_i\to X_{i+1}$ and $X_i\not\in \sZ$.
  \item Collider:  $X_{i-1}\to X_i\gets X_{i+1}$
    and some descendant of $X_i$ is in $\sZ$.
  \end{itemize}

  Two nodes $X$ and $Y$ are \emph{d-connected} by (conditioning on) a set $\sZ$ of
  nodes if there is an undirected path between $X$ and $Y$ that is active when
  conditioning on $\sZ$; otherwise $X$ and $Y$ are
  \emph{d-separated} by (conditioning on) $\sZ$.
  The notation $X\dsep Y\mid \sZ$ denotes d-separation 
  and $X\not\dsep Y\mid \sZ$ denotes d-connection.
  Note that paths of length 0 and 1 are always active, so a node is always d-connected
  to itself and to its parents and children.
\end{definition}

It has been shown that if $X$ and $Y$ are d-separated by $\sZ$,
then they are conditionally independent given $\sZ$
in any parameterization $P$ of the
graph \citep{Verma1988soundness}.
Conversely, if they are d-connected, then there is some parameterization $P$
in which they are conditionally
dependent given $\sZ$ \citep{Geiger1990completeness,Meek1995}.

\subsection{Causal Influence Diagrams}
\label{sec:influence-diagrams-sa}

\emph{Influence diagrams} are graphical models with special
decision and utility nodes, developed to model decision-making problems
\citep{Howard1984,Koller2003}.
This makes them good models for situations where an agent is trying to optimize
an objective in an environment.%
\footnote{In \posscite{Dennett1987} terminology,
  causal graphs can represent a \emph{physical stance}, while
  influence diagrams can be used to represent an \emph{intentional stance}.}
See \cref{fig:id-example} for an example.
We will use the term \emph{causal influence diagram} (CID) for influence diagrams
where all arrows encode causal relationships.%
\footnote{%
  In the influence diagram literature, a weaker causality condition applying only
  to descendants of decisions is often used \citep{Shachter2010,Heckerman1995}.  
}

As with causal graphs, we begin by defining the graph that
specifies only the structure of the interaction.

\begin{definition}[CID graph]
  \label{def:envg}
  A \emph{CID graph} is a tuple $G=\envg$, with
  \begin{itemize}
  \item $\cg$ a causal graph
  \item $\sA\subseteq \sW$
    an ordered set of \emph{decision nodes},
    represented by blue rectangles
    \begin{tikzpicture} 
      \node [draw, decision, minimum size=\ucht, inner sep=0mm] {};
    \end{tikzpicture}
  \item $\sU\subseteq \sW\setminus\sA$ a set of \emph{utility nodes},
    represented by yellow
    octagons
    \begin{tikzpicture}
      \node [draw, utilityc=1.5pt, minimum size=\ucht] {};
    \end{tikzpicture}.
  \item The remaining nodes $\sW\setminus(\sA\cup\sU)$ are called \emph{chance
      nodes}, and are represented with white
    circles
    \begin{tikzpicture} 
      \node [draw, circle,minimum size=\ucht, inner sep=0mm] {};
    \end{tikzpicture}
    or
    rectangles with rounded corners
        \begin{tikzpicture} 
      \node [draw, rectangle, rounded corners=2] {};
    \end{tikzpicture}.
\end{itemize}
  The parents $\Pa_D$ of a decision node $D\in\sA$ represent the \emph{decision context}
  for $D$, i.e.\ what information is available when $D$ is chosen.
  Information links $\Pa_D\to D$ are represented with dotted edges.
\end{definition}

\begin{figure}
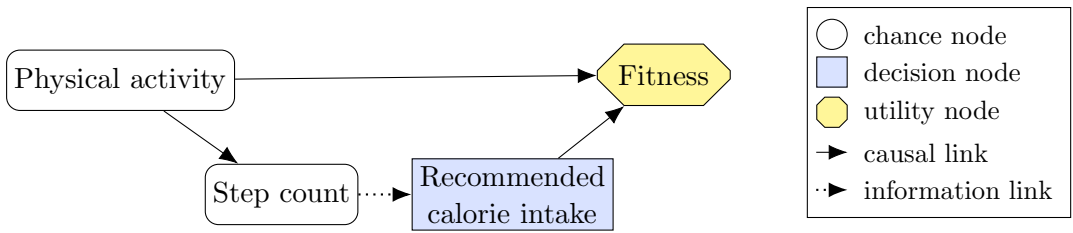

  \centering
  \begin{influence-diagram}

    \node (O) [] {Step count};
    \node (A) [decision, right = of O] {Recommended\\ calorie intake};

    \node (hA) [above = of A, draw=none] {};
    \node (Y) [right = of hA, utilityc={5pt}] {Fitness};

    \node (hS) [left = of O, draw=none] {};
    \node (S) [above = of hS] {Physical activity};

    \path
    (S) edge[->] (Y)
    (S) edge[->] (O)
    (A) edge[->] (Y)
    ;

    \path
    (O) edge[->, information] (A)
    ;

    \cidlegend[right = of Y, yshift=-5mm]{
      \legendrow{}{chance node} \\
      \legendrow{decision}{decision node}\\
      \legendrow{utilityc, chamfered rectangle xsep=1.5pt, chamfered rectangle ysep=1.5pt}{utility node}\\
      \legendrow[causal]{draw=none}{causal link} \\
      \legendrow[information]{draw=none}{information link} \\ 
    }

    \path (causal.west) edge[->] (causal.east);
    \path (information.west) edge[->, information] (information.east);
    
  \end{influence-diagram}
  \caption{
    Example of a CID.
    A machine learning system is recommending calorie intake (decision) to optimize the
    user's fitness (utility). The optimal calorie intake depends on the person's
    physical activity, which cannot be measured directly.
    Instead, the decision must be based on a step count provided by a
    fitness tracker.
  }
  \label{fig:id-example}
\end{figure}

\Cref{fig:id-example} shows a CID for a machine learning
system that uses step count as a proxy for physical activity to recommend ideal
calorie intake.
This setup will be our running example throughout the rest of the paper.
For an additional example, a Markov decision process with unknown
state transition function is modeled in \cref{sec:unknown}.

As with causal models, the precise relationship between the nodes is
specified with conditional probability distributions.
One important difference between CIDs and causal graphs is that a
CID parameterization only specifies conditional probability distributions for
non-decision nodes,
as the decisions are made exogenously to the model.

\begin{definition}[CID model]
  \label{def:envm}
  A \emph{CID model} is a tuple $M=\envm$ where
  \begin{itemize}
  \item $\envg$ is a CID graph
  \item For each node $X\in\sW$, the parameterization $P$ specifies:
    \begin{itemize}
    \item a finite domain $\dom(X)$; for utility nodes $X\in\sU$,
      the domain must be real-valued $\dom(X)\subset\SetR$
    \item conditional probability distributions $P(x\mid \pa_X)$
      for all non-decision nodes $X\in\sW\setminus\sA$.
    \end{itemize}
  \end{itemize}  
\end{definition}

In the influence diagram literature,
it is common to also require that utility nodes lack children and are
deterministic functions of their parents (e.g.\ \citealp{Koller2003}).
We will refrain from requiring this, as it is an unnecessary restriction that
makes it awkward to model some situations, such as the MDP in
\cref{sec:unknown}.

\paragraph{Policies and expected utility}
A policy $\pi$ describes the decisions of an agent, via conditional
probability distributions $\pi(d\mid \pa_D)$ for each decision node $D\in\sA$.
A parameterization $P$ combined with a policy $\pi$,
induces a joint distribution $P(\cdot\mid \pi)$ over $\sW$.
The goal of the agent is to choose a policy $\pi$ that maximizes the sum of the
utility variables.
Following the convention in reinforcement learning \citep{Sutton2018},
we call this the \emph{value} of $\pi$:

\begin{definition}[Value function]
  \label{def:value-sa}
  Let $\envm$ be a CID model.
  The \emph{value} or \emph{expected utility} of a policy $\pi$ is
  \(
    V^{\pi}
    = \EE\left[ \sum_{U\in\sU}U \mmid  \pi\right]
    \)
    where the expectation is with respect to $P(\cdot\mid \pi)$.
  An \emph{optimal policy} $\pi^*$ is a policy that optimizes $V^\pi$,
  with \emph{optimal value} $V^*=V^{\pi^*}$.
\end{definition}

\section{Observation Incentives}
\label{sec:observation-incentives}

This section will be devoted to the following question:
\begin{quote}
  Which nodes would a decision maker like to know the outcome of, or \emph{observe},
  before making a decision?
  That is, which nodes have a positive \emph{value of information} \citep{Howard1966}.
\end{quote}
Following an introductory example (\cref{sec:observation-example}),
we give a natural definition of observation incentive, and show that it can be
identified in any CID graph (\cref{sec:observation-theorem}).
An explanation of how to apply the theorem and interpret the result is given in
\cref{sec:observation-method}.
We conclude the section with an application to fairness
(\cref{sec:fairness-example}).

\subsection{Introductory Example}
\label{sec:observation-example}

Let us start by heuristically identifying%
\footnote{%
  \cref{th:observation-sa} below verifies all claims in this subsection.}
the observation incentives in an extension of the fitness tracker example from
\cref{fig:id-example}.
As before, a machine learning system recommends calorie intake for optimizing fitness
based on a step-count proxy for physical activity.
To make the example more interesting, we have now added 
a node for a noisy estimate of walking distance based
solely on the step count (\cref{fig:observation-example}).
We ask the question:
Which nodes would it be useful for the machine learning system to observe in order to
provide the most accurate calorie intake recommendation for the goal of
optimizing the user’s fitness?

\begin{figure}
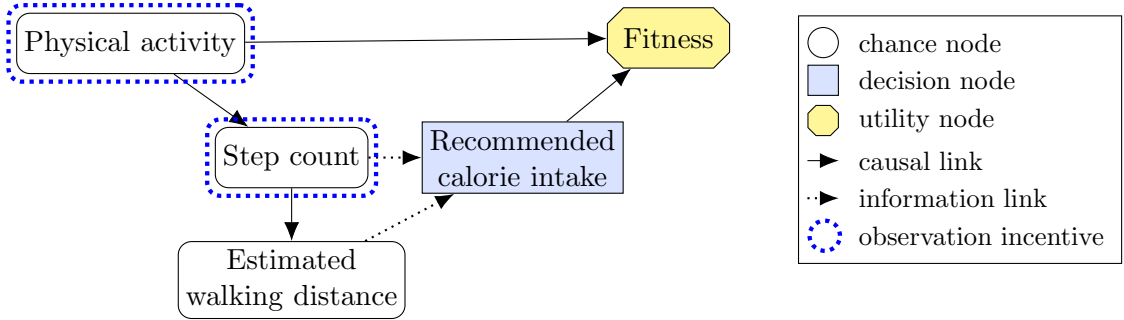

  \centering
  \begin{influence-diagram}

    \node (O) [] {Step count};
    \node (A) [decision, right = of O] {Recommended\\ calorie intake};

    \node (Op) [below = of O] {Estimated\\ walking distance};

    \node (hA) [above = of A, draw=none] {};
    \node (Y) [right = of hA, utilityc] {Fitness};

    \node (hS) [left = of O, draw=none] {};
    \node (S) [above = of hS] {Physical activity};

    \path
    (S) edge[->] (Y)
    (S) edge[->] (O)
    (A) edge[->] (Y)
    (O) edge[->] (Op)
    ;

    \path
    (O) edge[->, information] (A)
    (Op) edge[->, information] (A)
    ;
    \begin{scope}[ blue, dotted, ultra thick
      ]
      \node [fit = {(S)}] {};
      \node [fit = {(O)}] {};
    \end{scope}

    \cidlegend[right = of Y.north east, anchor=north west]{
      \legendrow{}{chance node} \\
      \legendrow{decision}{decision node}\\
      \legendrow{utilityc, chamfered rectangle xsep=1.5pt, chamfered rectangle ysep=1.5pt}{utility node}\\
      \legendrow[causal]{draw=none}{causal link} \\
      \legendrow[information]{draw=none}{information link} \\ 
      \legendrow{observation incentive}{observation incentive} \\
    }
    \path (causal.west) edge[->] (causal.east);
    \path (information.west) edge[->, information] (information.east);

  \end{influence-diagram}
  \caption{
    Observation incentives example.
    Here we return to the example of a machine learning system recommending calorie intake
    (\cref{fig:id-example}).
    To make it more interesting, we add a node for a (noisy)
    walking-distance estimate that is based solely on the step count.
    For deciding calorie intake, the step count but not the estimated walking
    distance provides useful information.
    The system also has an incentive to find out physical activity, even though it cannot measure it directly.
  }
  \label{fig:observation-example}
\end{figure}

First, it would be useful to observe physical activity, because
physical activity determines optimal calorie intake (by assumption).
In other words, there is an \emph{incentive} to observe physical activity.
Unfortunately, as the model is stated, it is not possible to observe physical activity
directly.
This makes the step count useful, because it can be used as a proxy
for physical activity.
In contrast, the estimate of the walking distance is not useful.%
\footnote{In the information-theoretic sense of the \emph{data processing
    inequality} \citep[Sec.~2.8]{Cover2006}.}
Even though it may contain information about the physical activity,
it cannot provide any additional information beyond the step count, because it is
only based on the step count in the first place.

Note that we do not ask the question whether the system wants to
observe the resulting fitness,
as it is a downstream effect of the decision.
Formally, observations of fitness are not permitted because they would introduce
cycles into the graph.

\subsection{Definition and Graphical Criterion}
\label{sec:observation-theorem}

If there is an observation incentive for a node $X$,
then the maximum expected utility should be strictly greater if
an information link $X\to D$ was present compared to if it was not.%
\footnote{Called a \emph{perfect observation} by \citet{Matheson2005}.}
It is straightforward to compare these two situations for a given CID model $M$,
because the parameterization $P$ only specifies conditional probability
distributions for non-decision nodes.
This means that the same $P$ can be kept
while information links are added or removed from the graph.

\begin{definition}[Single-decision observation incentive]
  \label{def:observation-incentive-sa}
  Let $M = \envmsa$ be a single-decision influence model and
  $X\in\sW\setminus\desc(D)$ a node not descending from $D$.
  Let $V^*_{X\to D}$ be the optimal value obtainable in $M$
  with an added information link $X\to D$, and 
  let $V^*_{X\not\to D}$ be the optimal value obtainable in $M$
  with any information link $X\to D$ removed.
  The agent has an
  \emph{observation incentive} for $X$
  if $V^*_{X\to D} > V^*_{X\not\to D}$. 
\end{definition}

As illustrated by the fitness tracker example in \cref{fig:observation-example},
what matters for observation incentives is whether a node carries information
about a utility node that can be influenced.
This can be assessed by a
d-separation criterion
(\cref{def:d-sep})
conditioned on the available information $\Pa_D$ and $D$.
For example, in \cref{fig:observation-example},
step count provides useful information while estimated walking distance does not.
This is explained by step count being d-connected to fitness via physical activity,
while estimated walking distance is d-separated from fitness because step count is observed.
Using this d-separation criterion, 
we can tell whether a CID graph is \emph{compatible} with an observation incentive
on a node $X$, i.e.\ whether observing $X$ would be useful under some
parameterization of the graph.

\begin{theorem}[Single-decision observation incentive criterion]
  \label{th:observation-sa}
  Let $\envgsa$ be a single-decision CID graph, and let
  $X\in\sW\setminus\desc(D)$ be a node not descending from the decision $D$.
  There exists a parameterization $P$ for $G$ in which the agent has an
  observation incentive for $X$
  if and only if $X$ is d-connected to a utility node that descends from $D$:
  \[X\not\dsep \sU \cap \desc(D) \mid \{D\} \cup \Pa_{D}\setminus \{X\} .\]
\end{theorem}

The theorem follows from \cref{th:soundness-sa,th:completeness-sa} in
\cref{ch:observation-proofs}.
The \emph{only if} part of the statement have previously been shown by
\citet{Fagiuoli1998,Lauritzen2001},
though they focused on a subset of our question:
namely which observed nodes $O\in\Pa_D$ are compatible with an observation incentive,
i.e.\ which observations are useful (\emph{requisite}) and not.%
\footnote{
  \citet{Lauritzen2001} also show how the \emph{only if} part of
  the criterion is extended to CIDs with multiple decision nodes.
}
In contrast, our interest is equally in which unobserved nodes
the agent would like to observe.
Nonetheless, some terminology for observation incentives in the decision context
$\Pa_D$ will be useful.

\begin{definition}[Requisite observations]
  \label{def:requisite-observation}
  An observation $O\in\Pa_D$ is a \emph{requisite observation}
  if it satisfies the observation incentive criterion (\cref{th:observation-sa}).
  Let $\Pa^*_{D}\subseteq \Pa_{D}$ denote the set of requisite observations.
  The rest of the observations $\Pa_{D}\setminus\Pa^*_{D}$ are \emph{nonrequisite}.
  Extending the terminology to information links,
  an information link $\Pa^*_{D} \to D$ is \emph{requisite},
  and an information link $(\Pa_D\setminus\Pa^*_{D}) \to D$ is \emph{nonrequisite}.
\end{definition}

Since an optimal decision need not depend on nonrequisite observations,
for many purposes we can remove the information links from these nodes.
The reduced graph will be important for analyzing intervention incentives
(\cref{sec:intervention-incentives}),
as well as observation incentives in multi-decision and
multi-agent CID graphs (Part II and \citealp{Lauritzen2001}).

\begin{definition}[Reduced graph]
  \label{def:reduced-graph}
  The \emph{reduced graph} $G^*$ of a single-decision CID graph $G$
  is the result of removing all nonrequisite information links from $G$.
\end{definition}

\subsection{How to Use and Interpret the Criterion}
\label{sec:observation-method}

\paragraph{Method}
Concretely, the observation incentive criterion can be applied per the
following.
To check whether a node $X$ may face an observation incentive,
begin by checking whether $X$ is a descendant of $D$.
Only if it is not a descendant can we enquire about its observation incentives.
If it is not a descendant of $D$, then
check whether it is d-connected to $U\cap\desc(D)$ when conditioning on
$D$ and $\Pa_D $ but not $X$ with the following procedure:

Begin by \emph{marking} the nodes $D$ and $\Pa_D\setminus \{X\}$
as nodes to be conditioned on.
There is an observation incentive for $X$ if and only if
it is possible to:
\begin{enumerate}
\item Go forward%
  \footnote{\emph{Going forward} means following the arrows,
    and \emph{going backwards} means going in the reverse direction.
  }
  from $D$ to a utility node $U\in\sU$
\item Starting from $U$, it is possible to reach $X$ using the following rules:
\begin{enumerate}
\item Go backwards without passing any marked node.
  At any point, switch to step b.
\item Go forward without passing any marked node.
  When reaching a marked node, switch to step a.
\end{enumerate}
\end{enumerate}
An intuitive way of thinking about the procedure is that paths can ``bounce
forward'' from unmarked nodes, and ``bounce backward'' from marked nodes.%
\footnote{
  For this reason, the procedure has been called the \emph{Bayes ball} algorithm
  \citep{Shachter1998},
though maybe \emph{Bayes\emph{ket} ball} would have been an even more appropriate
name for the procedure?
}
It is not necessary that the path ever bounces for there to be an observation
incentive for $X$.

\begin{figure}
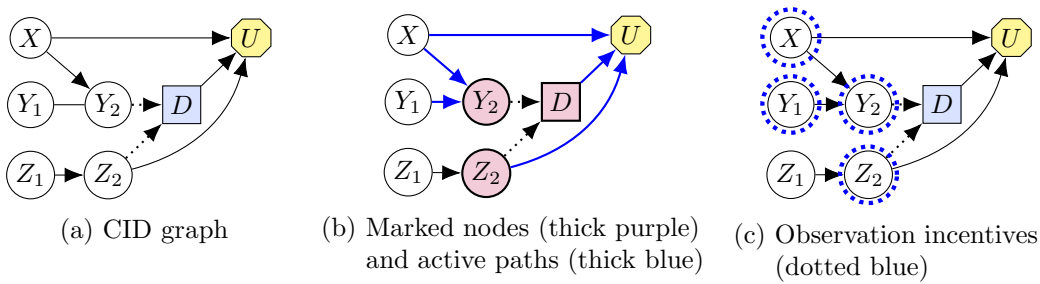

  \centering
  \begin{subfigure}{0.3\textwidth}
    \centering
    \begin{influence-diagram}
      \setcompact[every node/.append style = {circle}]

    \node (A) [decision] {$D$};

    \node (Y2) [left = of A] {$Y_2$};
    \node (Y1) [left = of Y2] {$Y_1$};

    \node (X) [above = of Y1] {$X$};

    \node (Z2) [below = of Y2] {$Z_2$};
    \node (Z1) [below = of Y1] {$Z_1$};

    \node (hU) [right = of A, draw=none] {};
    \node (U) at (hU |- X) [utility] {$U$};

    \path
    (A) edge[->] (U)
    (X) edge[->] (U)
    (X) edge[->] (Y2)
    (Y1) edge[-] (Y2)
    (Z1) edge[->] (Z2)
    (Z2) edge[->, bend right] (U)
    ;

    \path
    (Y2) edge[->, information] (A)
    (Z2) edge[->, information] (A)
    ;
  \end{influence-diagram}
  \caption{CID graph\\ \phantom{asdflj l}}
  \label{fig:observation-method-or}
\end{subfigure}
\begin{subfigure}{0.36\textwidth}
  \centering
  \begin{influence-diagram}
    \setcompact[every node/.append style = {circle}]

    \node (A) [decision, fill=purple!20, thick] {$D$};

    \node (Y2) [left = of A, fill=purple!20, thick] {$Y_2$};
    \node (Y1) [left = of Y2] {$Y_1$};

    \node (X) [above = of Y1] {$X$};

    \node (Z2) [below = of Y2, fill=purple!20, thick] {$Z_2$};
    \node (Z1) [below = of Y1] {$Z_1$};

    \node (hU) [right = of A, draw=none] {};
    \node (U) at (hU |- X) [utility] {$U$};

    \path
    (A) edge[->, thick, blue] (U)
    (X) edge[->, thick, blue] (U)
    (X) edge[->, thick, blue] (Y2)
    (Y1) edge[->, thick, blue] (Y2)
    (Z1) edge[->] (Z2)
    (Z2) edge[->, bend right, thick, blue] (U)
    ;

    \path
    (Y2) edge[->, information] (A)
    (Z2) edge[->, information] (A)
    ;
  \end{influence-diagram}
  \caption{Marked nodes (thick purple)\\ and active paths (thick blue)}
  \label{fig:observation-method-marked}
\end{subfigure}
\begin{subfigure}{0.3\textwidth}
  \centering
  \begin{influence-diagram}
    \setcompact[every node/.append style = {circle}]

    \node (A) [decision] {$D$};

    \node (Y2) [left = of A] {$Y_2$};
    \node (Y1) [left = of Y2] {$Y_1$};

    \node (X) [above = of Y1] {$X$};

    \node (Z2) [below = of Y2] {$Z_2$};
    \node (Z1) [below = of Y1] {$Z_1$};

    \node (hU) [right = of A, draw=none] {};
    \node (U) at (hU |- X) [utility] {$U$};

    \path
    (A) edge[->] (U)
    (X) edge[->] (U)
    (X) edge[->] (Y2)
    (Y1) edge[->] (Y2)
    (Z1) edge[->] (Z2)
    (Z2) edge[->, bend right] (U)
    ;

    \path
    (Y2) edge[->, information] (A)
    (Z2) edge[->, information] (A)
    ;

    \begin{scope}[observation incentive,
      every node/.append style = {inner sep=0.5mm, minimum size=0.75cm},
      ]
      \node at (X) {};
      \node at (Y2) {};
      \node at (Y1) {};
      \node at (Z2) {};
    \end{scope}
  \end{influence-diagram}
  \caption{Observation incentives\\ (dotted blue)}
  \label{fig:observation-method-done}
\end{subfigure}
\caption{How to use the observation incentive criterion.}
\end{figure}

For example, in the graph shown in \cref{fig:observation-method-or}, we begin by
marking the nodes $D$ and $\Pa_D=\{Y_2, Z_2\}$
(\cref{fig:observation-method-marked}).
Then we start at $D$, and reach $U$ in a single step.
There are three ways to go backwards from $U$: to $X$, to $D$, and to $Z_2$.
The \emph{active paths} they give rise to are illustrated
with thick blue paths in \cref{fig:observation-method-marked}.
Let us consider these in turn.
The topmost path (to $X$) can ``bounce'' forward again at $X$, since $X$ is an
unmarked fork node.
From $X$ we can go forward to $Y_2$, which is a marked node, and therefore
allows us to ``bounce'' backwards again, to $Y_1$.
From $Y_1$ we can go no further however, and we have exhausted the possible
paths arising from $X$.
The middle path (to $D$) only reaches $D$, which is a descendant of $D$ and
therefore disregarded.
Since $D$ is marked, the path stops here.
The bottommost path to $Z_2$ reaches $Z_2$. Since $Z_2$ is marked, the path stops
here.
The nodes that are not a descendant of $D$ and that have been reached through one
of these paths are the nodes facing an observation incentive; see
\cref{fig:observation-method-done}.

\paragraph{Interpretation}
Once we know whether there is an observation incentive for a node $X$, we need to know
how to interpret the result.
Observation incentives have slightly different interpretations for observed and
unobserved nodes.
For observed nodes $X\in\Pa_D$, an observation incentive simply means that the
agent's optimal decision may depend on $X$, as with step count in
\cref{fig:observation-example}.
In other words, the node is \emph{requisite} for an optimal decision.
For unobserved nodes $Y\not\in\Pa_D$, an observation incentive means that an
agent with additional access to $Y$ may be able to achieve higher expected utility.
In practice, this can mean that the agent (partially) infers $Y$ from
information that it does have access to. 
A good example of this is physical activity in \cref{fig:observation-example},
which is partially inferred from step count.
In situations where the model is only an approximation of reality, it can also
mean that the agent finds a way to directly observe $Y$.
Examples of this could be a poker player that takes a sneak peak at his
opponents cards, or a company that orders an extra market analysis before making
a decision.

\subsection{Application to Fairness}
\label{sec:fairness-example}

Let us see how observation incentives can be applied in questions of
fairness and discrimination \citep{ONeil2016}.
One type of discrimination is \emph{disparate treatment} \citep{Barocas2016},
which occurs when a decision process treats people differently based on
\emph{sensitive attributes} such as race or gender.
However, what this means formally is still subject to intense debate
(e.g.\ \citealp{Corbett-Davies2018,Gajane2017}).
In this section, we illustrate how observation incentives for sensitive
attributes can contribute to this discussion.

As an example, we will consider the Berkeley admission case \citep{Bickel1975}.
In this case, it was found that the admission rate for men was higher than for women.
However, the difference in admission rate was explained by women applying to more
competitive departments than men.
Was the university guilty of discriminating against women?

A nuanced account of the situation can be obtained using causal graphs
\citep{Pearl2009,Mancuhan2014,Bonchi2017,Chiappa2019,Kilbertus2017,Kusner2017,Zhang2017}.
Using a causal graph similar to the one represented in \cref{fig:fair1},
\citet{Pearl2009} argues that since the influence from gender to admission
was mediated by department choice, the university was not discriminating against
women.
An assumption in Pearl's argument is that the university was using the
applicant's department choice to fit the right number of students into each
department.
This assumption can be made explicit in the
\emph{path-specific counterfactual fairness} framework \citep{Chiappa2019},
where causal pathways from sensitive attributes to decision nodes are labeled
\emph{fair} or \emph{unfair}.
For example, the path from gender to admission would be considered fair if
department choice was used to fit the right number of students into each
department, and unfair if the university was using department choice to covertly
gender bias the student population by lowering the admission rate for
departments that women were more likely to apply to.

\begin{figure}[t]
  \centering
  \begin{subfigure}{0.48\textwidth}
  \centering
  \begin{influence-diagram}[node distance=0.5cm and 2.5cm]
    \node (a) {Admit?};
    \node (d) [left = of a] {Department\\ choice};
    \node (g) [above = of d] {Gender};

    \path
    (g) edge[->] (d)
    ;

    \begin{scope}{node distance = 0.1mm}
      \draw[->] (d) -- (a) node[midway,below,draw=none,yshift=0mm] {fair/unfair?};
    \end{scope}

  \end{influence-diagram}
  \caption{Causal graph with path-label in the path-specific
    counterfactual fairness framework \\ \phantom{asdfs}}
  \label{fig:fair1}
\end{subfigure}
\begin{subfigure}{0.48\textwidth}
  \raggedleft
  \begin{tikzpicture}
    \cidlegend{
      \legendrow{}{chance node} \\
      \legendrow{decision}{decision node}\\
      \legendrow{utilityc, chamfered rectangle xsep=1.5pt, chamfered rectangle ysep=1.5pt}{utility node}\\
      \legendrow[causal]{draw=none}{causal link} \\
      \legendrow[information]{draw=none}{information link} \\ 
      \legendrow{observation incentive}{observation incentive} \\
    }

    \path (causal.west) edge[->] (causal.east);
    \path (information.west) edge[->, information] (information.east);
    
  \end{tikzpicture}
\end{subfigure}
  \begin{subfigure}[t]{0.48\textwidth}
  \centering
  \begin{influence-diagram}\setcompact[node distance=6.5mm and 0.6cm,
    every node/.append style = {minimum height=0.6cm}
    ]
    \node (a) [decision] {Admit?};
    \node (d) [left = of a] {Department \\ choice};
    \node (g) [above = of d] {Gender};
    \node (gpa) [right = of a, utilityc] {Student \\ performance};
    \node (pos) at (gpa |- g) [utilityc] {Right $\#$ of\\ students per \\
      department};

    \path
    (g) edge[->] (d)
    (d) edge[->, information] (a)
    (a) edge[->] (gpa)
    (a) edge[->] (pos)
    (d.north east) edge[->] (pos)
    ;

    \begin{scope}[ observation incentive
      ]
      \node [fit = {(d)}] {};
    \end{scope}
    
  \end{influence-diagram}
  \caption{CID graph of unbiased university}
  \label{fig:fair2}
\end{subfigure}
  \begin{subfigure}[t]{0.48\textwidth}
  \centering
  \begin{influence-diagram}
    \setcompact[
    node distance=6.5mm and 0.6cm,
    every node/.append style = {minimum height=0.6cm}]
    \node (a) [decision] {Admit};
    \node (d) [left = of a] {Department \\ choice};
    \node (g) [above = of d] {Gender};
    \node (gpa) [right = of a, utilityc] {Student\\ performance};
    \node (bias) at (gpa |- g)  [utilityc] {$\%$ men};

    \path
    (g) edge[->] (d)
    (d) edge[->, information] (a)
    (a) edge[->] (gpa)
    (a) edge[->] (bias)
    (g) edge[->] (bias)
    ;

    \begin{scope}[ observation incentive ]
      \node [fit = {(d)}] {};
      \node [fit = {(g)}] {};
    \end{scope}
  \end{influence-diagram}
  \caption{CID graph of gender-biased university}
  \label{fig:fair3}
\end{subfigure}
\caption{Graphical representations of the Berkeley admission case
  \citep{Bickel1975}.
  }
  \label{fig:fairness}
\end{figure}
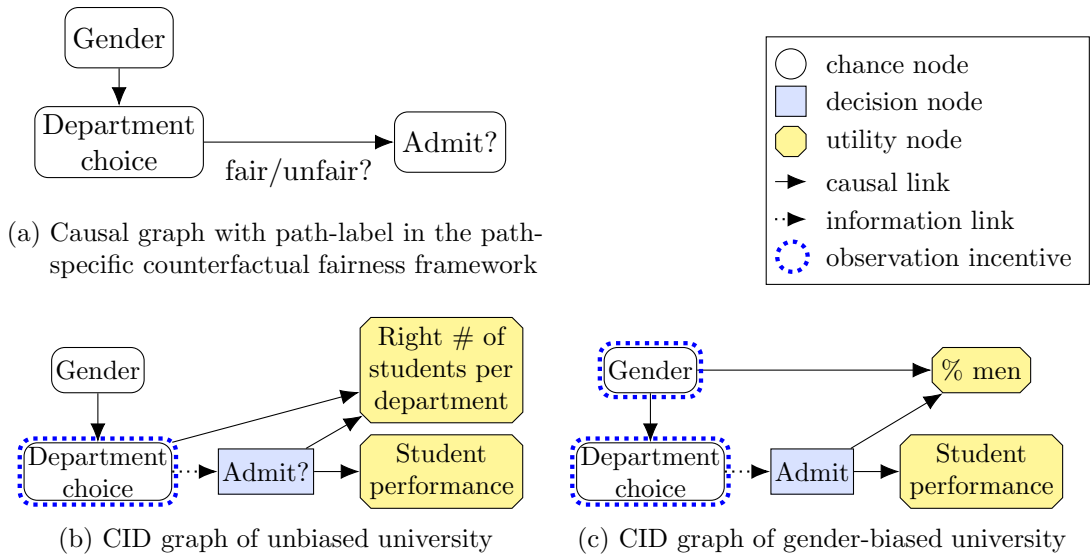

Observation incentives offer an alternative to path-labeling for judging
disparate treatment.
Universities can be modeled as agents that choose which students to admit in
order to optimize an objective function such as student performance.
Consider the CIDs in \cref{fig:fair2,fig:fair3}
of two universities that have different additional objectives beside student performance.
The university in \cref{fig:fair2} tries to fit the right number of students into each
department; the university in \cref{fig:fair3} covertly tries to gender bias the
student population by using department choice as a proxy for gender.
As both universities only use department choice for the decision,
the causal pathway from gender to admission is the same in both cases.

\begin{samepage}
We can use observation incentives to explain the difference in fairness between
the universities,
from the different information they infer from department choice:
\begin{itemize}
\item The first university has no observation incentive for gender. It is only
  using the department choice to fit the right number of students into each
  department. 
\item The second university may have an observation incentive for gender. It
  may therefore be using department choice to infer the gender of the student,
  which may render it guilty of disparate treatment.
\end{itemize}
\end{samepage}

The need to know the objectives of the decision maker somewhat limits the
applicability of incentives-based fairness approaches.
For example, an outsider may be unable to find out the objectives of the
universities in the above example.
This difficulty is resembles the difficulty of correctly labeling paths
fair or unfair in the path-specific counterfactual fairness approach.
However, an advantage with the observation incentive approach is that when
we are training a machine learning system, then we are aware of what objective
function the system is optimizing, and what information the system has access
to.
Combined with a CID for how the objective and the observed
information interacts with the sensitive attributes,
the observation incentive criterion can be used to identify which
incentives emerge from this objective, and whether they involve problematic
inference of sensitive attributes.

\pagebreak
\section{Intervention Incentives}
\label{sec:intervention-incentives}

This section asks the question:
\begin{quote}
  Which nodes would an agent like to influence or control?
  That is, which nodes face a positive \emph{value of control} \citep{Shachter2010}.
\end{quote}
Building on the observation incentive criterion,
we establish an analogous \emph{intervention incentive} criterion
(\cref{sec:intervention-incentive-def}), and explain how
to use an interpret it (\cref{sec:intervention-method}).
The section concludes with an application to the incentives of
QA systems (\cref{sec:oracle-example}).

\subsection{Introductory Example}
\label{sec:intervention-example}

\begin{figure}
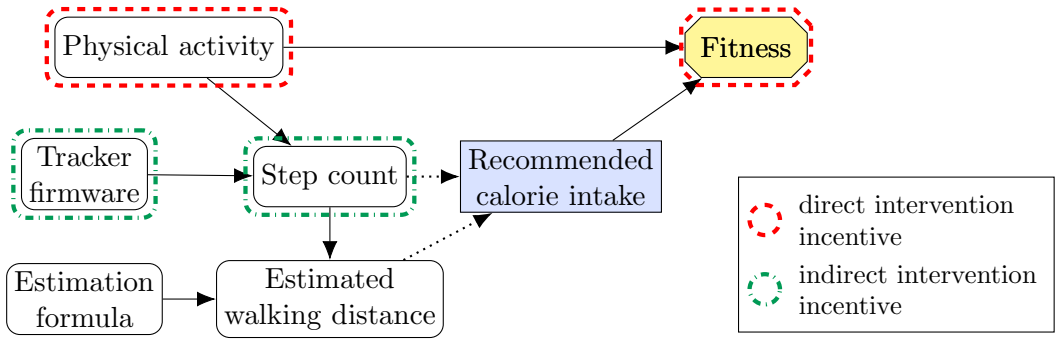

  \centering
  \begin{influence-diagram}
    \node (O) [] {Step count};
    \node (A) [decision, right = of O] {Recommended \\ calorie intake};

    \node (Op) [below = of O] {Estimated\\ walking distance};

    \node (hS) [left = of O, draw=none] {};
    \node (S) [above = of hS, yshift=2mm] {Physical activity};

    \node (Xp) [left = of Op] {Estimation\\ formula};
    \node (X) [above = of Xp] {Tracker\\ firmware};

    \node (hA) [right = of A, draw=none] {};
    \node (Y) at (hA |- S) [utilityc] {Fitness};

    \path
    (S) edge[->] (Y)
    (S) edge[->] (O)
    (A) edge[->] (Y)
    (O) edge[->] (Op)
    (X) edge[->] (O)
    (Xp) edge[->] (Op)
    ;

    \path
    (O) edge[->, information] (A)
    (Op) edge[->, information] (A)
    ;

    \begin{scope}[intervention incentive, 
      ]
      \node [fit = {(S)}] {};
      \node at (Y) [chamf, inner sep=1.5mm, minimum size=10mm] {\textcolor{black}{Fitness}};
    \end{scope}

    \begin{scope}[information incentive,
      ]
      \node [fit = {(O)}] {};
      \node [fit = {(X)}] {};
    \end{scope}

    \cidlegend[right = of A, anchor=north west]{
      \legendrow{intervention incentive}{direct intervention\\ incentive} \\
      \legendrow{information incentive}{indirect intervention\\ incentive}
    }

  \end{influence-diagram}
  \caption{
    Intervention incentives example.
    As in
    \cref{fig:id-example,fig:observation-example},
    a machine learning system uses an activity tracker
    to recommend calorie intake for optimizing fitness.
    Interventions can contribute utility either \emph{directly} by influencing a
    utility node, or \emph{indirectly} by increasing the information available
    at a decision.
    An example of the former kind would be to increase physical activity to improve
    fitness.
    An example of the latter kind is upgrading the tracker firmware to
    make the step count more accurate, as it would enable a
    more informed recommendation of calorie intake.
    In contrast to both of these types,
    improving the estimate of the walking distance would have no value at all,
    since the estimated walking distance is not a requisite observation.
  }
  \label{fig:intervention-example}
\end{figure}

Continuing the example from \cref{sec:observation-example},
let us also heuristically identify%
\footnote{All claims made in this subsection are verified by \cref{th:soft-sa} below.}
intervention incentives in the CID graph in \cref{fig:intervention-example}.
As before, a machine learning system recommends calorie intake for optimizing fitness
based on information provided by fitness tracker.
We ask the question:
\emph{Which nodes would be useful to influence in addition to the calorie intake?}
In other words, influence over which nodes would enable the system to optimize
its utility?

Trivially, the system would like to control fitness, since that is its
optimization target.
Similarly, influencing physical activity means indirectly controlling fitness,
and would therefore be useful as well.
The situation is more subtle with the ancestors of calorie intake.
To start with, the only benefit of step count is its informativeness about
physical activity.
This means that interventions that increase the accuracy of step count
are useful.
An example of such an intervention is to update the tracker firmware.
In contrast, interventions on estimated walking distance are never useful, as it is
not used in an (optimal) decision for calorie intake anyway, as discussed in
\cref{sec:observation-example}.

\subsection{Definition and Graphical Criterion}
\label{sec:intervention-incentive-def}

Our definition of intervention incentive is analogous to the definition of
observation incentive (\cref{def:observation-incentive-sa}).
Instead of considering observing an extra node, we consider controlling an extra node,
where control is formalized with soft interventions:

\begin{definition}[Soft intervention; \citealp{Eberhardt2007}; {\citealp[p.~74]{Pearl2009}}]
  \label{def:soft-intervention}
  A \emph{(soft) intervention} $c^X$ on a non-decision node $X$
  in a CID model $M = \envmsa$
  changes the conditional probability distribution for $X$ from $P(x\mid \pa_X)$ to
  $c^X(x\mid \pa_X)$, while leaving all other conditional probability
  distributions intact.%
  \footnote{
    It is sometimes more natural to think of soft interventions
    as changing the relation between $\Pa_X$ and $X$, rather than changing $X$ directly.
    However, following the convention in the literature, we will speak of them as
    interventions on the node $X$ and nothing else.}
  We write $P(\cdot \mid c^X)$ for the updated probability distribution.
\end{definition}

Control can also be formalized by adding extra decision nodes
\citep{Matheson2005,Shachter2010}.
Indeed, soft interventions correspond to a probabilistic generalization
of \emph{perfect control} \citep{Matheson2005},
and to \emph{atomic interventions on a mapping variable} \citep{Shachter2010}.

\begin{definition}[Single-decision intervention incentive]
  \label{def:intervention-incentive-sa}
  Let $M = \envmsa$ be a single-decision CID model and $X$ a non-decision node
  $X\in \sW\setminus\{D \}$.
  Let $V^{\pi,c^X} = \EE\left[\sum_{U\in\sU}U\mmid \pi, c^X\right]$ be the value
  of following policy $\pi$ and controlling $X$ with intervention~$c^X$.
  The agent has an \emph{intervention incentive} on $X$ 
  if $\max_{\pi,c^X}V^{\pi,c^X} > \max_{\pi'}V^{\pi'}$.
\end{definition}

Similarly to observation incentives,
a graphical criterion can tell us whether a CID graph is compatible with an
intervention incentive on a non-decision node $X$. 

\begin{theorem}[Single-decision intervention incentive criterion]
  \label{th:soft-sa}
  Let $X\in\sW\setminus \{D\}$ be a non-decision node
  in a single-decision CID graph $G = \envgsa$.
  There exists a parameterization $P$ for $G$ such that
  the agent has an intervention incentive for $X$
  if and only if there is a directed path
  $X\pathto \sU$ in the reduced graph $G^*$.
\end{theorem}

The intuition for the criterion is that only if there is a path from $X$ to a
utility node can intervening on $X$ have any effect on the utility of the
agent.
Note that the criterion uses the reduced graph $G^*$
where nonrequisite information links have been cut
(\cref{def:reduced-graph}),
because nonrequisite observations do not affect the optimal decision,
and therefore cannot propagate the effect of the intervention.
A proof of the criterion can be found in 
\cref{sec:intervention-proofs}.

\paragraph{Types of intervention incentives}
Note that the path $X\pathto U$ in \cref{th:soft-sa}
is allowed to pass through the decision $D$.
The question of whether it does, allows us to distinguish between two
different reasons the agent wants to intervene on $X$:
\begin{itemize}
\item
  The path $X\pathto \sU$ yields a \emph{direct%
    \footnote{The intervention incentive is \emph{direct} in the sense that it does
      not pass $D$. The effect from $X$ to $\sU$ may still be mediated by other
      variables.}
    intervention incentive} on $X$
  if the path does not pass $D$.
\item
  The path $X\pathto \sU$ yields an
  \emph{indirect intervention incentive} on $X$, if the path passes $D$
  and there is also another path $X\upathto\sU$ that is \emph{not} directed 
  and is active when conditioning on $\Pa_D\cup\{D\}$.
\end{itemize}
Extending the terminology, we say that there is an
\emph{(in)direct intervention incentive on $X$}
if there is a path yielding an (in)direct intervention incentive on $X$.
We will also speak of direct intervention incentive
as incentives \emph{for direct control} and indirect ones as incentives
\emph{for information}.
For example, the intervention incentives for step count and tracker firmware in
\cref{fig:intervention-example} are for information, whereas the
intervention incentives for physical activity are for direct control.
Note that the reasons are not mutually exclusive:
it is possible that an intervention can simultaneously provide both direct control
and information, if it is connected to utility nodes via several paths.
However, the types are collectively exhaustive:
if a node faces an intervention incentive, then it faces either a direct or an
indirect intervention incentives (or both).
In particular, if there is a path $X\pathto D\pathto U$ but the path fails
to provide an indirect intervention incentive, then there must also be a path
$X\pathto U$ not passing $D$, providing a direct intervention incentive for $X$.

\subsection{How to Use and Interpret the Criterion}
\label{sec:intervention-method}

\paragraph{Method}
To apply the intervention incentive criterion,
first cut all nonrequisite information links.
To do this, follow the procedure described in \cref{sec:observation-method} to
determine which observations face an observation incentive, and remove the
information links from those without observation incentive.
Once we have removed all nonrequisite information links and obtained the reduced
graph $G^*$, it is straightforward to assess intervention incentives:
there is an intervention incentive on a node $X$ if and only if $X$ is not the
decision node and
there is a directed path from $X$ to a utility node $U\in\sU$ in
the reduced graph $G^*$.

For example, in the fitness tracker example in \cref{fig:intervention-example},
the information link from estimated walking distance to calorie intake will be
cut as it is nonrequisite.
After that, there is no directed path from estimated walking distance to the
utility node fitness, which means that there is no intervention incentive
on estimated walking distance.
In contrast, the information link from step count to calorie intake is not cut
because it is requisite.
Therefore a directed path remains to fitness, which means that there is a
intervention incentive for step count and tracker firmware.

\begin{figure}
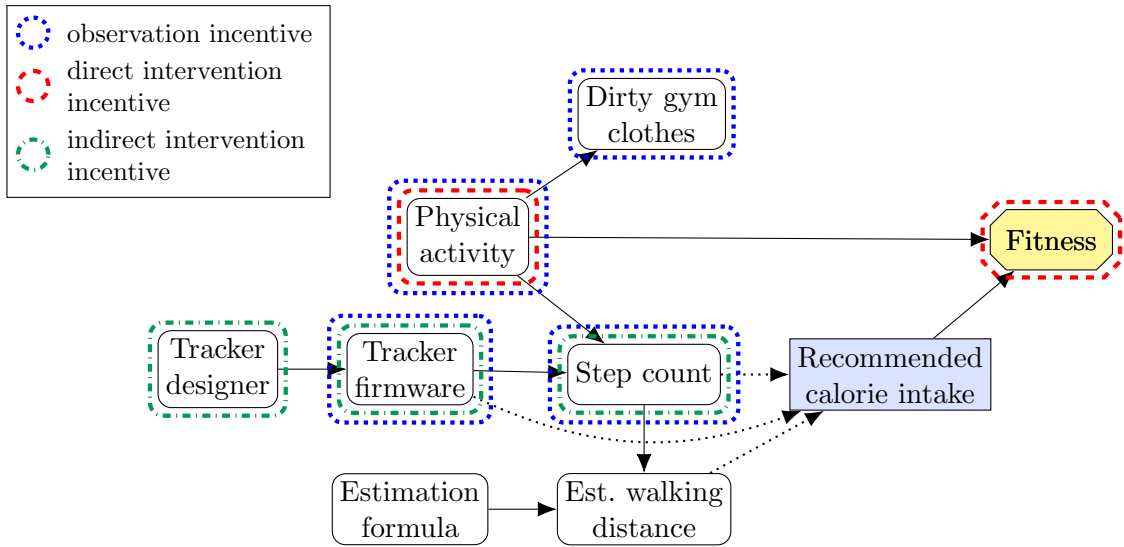

  \centering
  \begin{influence-diagram}[node distance=0.9cm]
    \node (O) [] {Step count};
    \node (A) [decision, right = of O] {Recommended\\ calorie intake};

    \node (Op) [below = of O] {Est.\ walking\\ distance};

    \node (hA) [above = of A, draw=none] {};
    \node (Y) [right = of hA, utilityc] {Fitness};

    \node (hS) [left = of O, draw=none] {};
    \node (S) [above = of hS] {Physical\\ activity};

    \node (Xp) [left = of Op] {Estimation\\ formula};
    \node (X) [above = of Xp] {Tracker\\ firmware};

    \node (Xpp) [left = of X] {Tracker\\ designer};

    \node (Sp) [above right = of S] {Dirty gym\\ clothes};

    \path
    (S) edge[->] (Y)
    (S) edge[->] (O)
    (A) edge[->] (Y)
    (O) edge[->] (Op)
    (X) edge[->] (O)
    (Xp) edge[->] (Op)
    (Xpp) edge[->] (X)
    (S) edge[->] (Sp)
    ;

    \path
    (O) edge[->, information] (A)
    (Op) edge[->, information] (A)
    (X) edge[->, information, bend right=23] (A)
    ;

    \begin{scope}[intervention incentive, inner sep=0.1mm]
      \node (Si) [fit = {(S)}] {};
      \node at (Y) [chamf, inner sep=2mm, minimum height=10mm] {\textcolor{black}{Fitness}};
    \end{scope}

    \begin{scope}[information incentive, inner sep=0.1mm]
      \node [fit = {(Xpp)}] {};
      \node (Xi) [fit = {(X)}] {};
      \node (Oi) [fit = {(O)}] {};
    \end{scope}

    \begin{scope}[observation incentive, inner sep=0.1mm]
      \node [fit = {(Sp)}] {};
      \node [fit = {(Xi)}] {};
      \node [fit = {(Oi)}] {};
      \node [fit = {(Si)}] {};
    \end{scope}

    \cidlegend[left = of S.north west, anchor=south east, node distance=5mm]{
      \legendrow{observation incentive}{observation incentive} \\
      \legendrow{intervention incentive}{direct intervention\\ incentive}
      \legendrow{information incentive}{indirect intervention\\ incentive}
    }
  \end{influence-diagram}
  \caption{
    Examples where observation incentives and intervention incentives
    deviate in a variant of the examples from
    \cref{fig:id-example,fig:intervention-example,fig:observation-example}.
    If the fitness tracker firmware is fully known, additional information about the
    tracker designer is not useful, so there is no observation incentive for
    tracker designer.
    But having been able to improve the tracker designer's design abilities would have
    been useful, as it could have resulted in a better tracker.
    Thus, there is an indirect intervention incentive on the tracker designer.
    In contrast,
    a side effect of (some types of) physical activity is dirty gym
    clothes.
    There is no point controlling dirty gym clothes, because making the gym
    clothes dirty by other means than physical activity will will not cause fitness.
    But observing whether the gym clothes are dirty would give some additional
    information about physical activity not necessarily present in the step
    count (especially if the tracker is not worn in the gym).
  }
  \label{fig:obs-vs-int}
\end{figure}

\paragraph{Interpretation}
Assume that we have established an intervention incentive for a node $X$.
How should we now interpret this?
If $X$ is a utility node, then trivially the agent wants to influence $X$, which
we already knew.
If $X$ is a non-utility node that is a descendant of some of the agent's
decision nodes, then an intervention incentive on $X$ suggests that the
agent may use its decision to control $X$ as an \emph{instrumental goal} in
order to ultimately gain some utility from it.
Finally, if $X$ is a not a descendant of any of the agent's decision nodes, then if the
model is to be interpreted literally, there is nothing the agent can do about
$X$.
We may wish that gravity was less strong, but there is not much we can do about
fundamental physical constants.

However, in many cases, the model is only an approximation of reality.
For example, a worry in the AI safety literature \citep{Everitt2018litrev}
is that an agent finds a way to
tamper with the reward signal, giving itself high reward without completing its
intended goals. 
Indeed, it has been demonstrated that the Super Mario game environment can
be made to run arbitrary code by selecting the right decision sequences \citep{Masterjun2014}.
This could in principle be used by the agent to hack the reward function to
maximize the reward without completing the game.
Such influences may break the designer's assumptions about how the agent can
influence the environment, and has been modeled with CIDs by
\citet{Everitt2019tampering}.

\paragraph{Comparison to observation incentives}
In many cases, nodes face either both an observation incentive and an
intervention incentive, or neither.
However, there are a few of notable cases where the incentives diverge.
\Cref{fig:obs-vs-int} shows a few of them.

\subsection{Application to Question-Answering Systems}
\label{sec:oracle-example}

In \emph{Superintelligence}, \citet{Bostrom2014} discusses different ways to use
powerful artificial intelligence.
One possibility is to let an agent continuously interact with the world to achieve
some long-term goal.
Another possibility is to construct a pure question-answering system 
(QA-system), with the only goal to correctly answer queries \citep{Armstrong2012oracles}.
QA-Systems have some safety benefits, as they only affect the world through their
answers to queries and can be constructed to lack long-term goals.

One safety concern with QA-systems is the following.
Assume that we ask our QA-system about the price of a particular stock one week
from now, in order to make some easy money trading it.
Then the answer will affect the world, because anyone who knows the QA-system's
answer will factor it into his or her trading decisions.
This effect may be enough to make the answer wrong, even if the answer would
have been right had no one heard of it.
More worryingly perhaps, the answer may also become a \emph{self-fulfilling prophecy}.
A respected QA-system that predicts the bankruptcy of a company within a week,
may cause the company to go bankrupt if the prediction leads to investors and
other stakeholders losing confidence in the business.

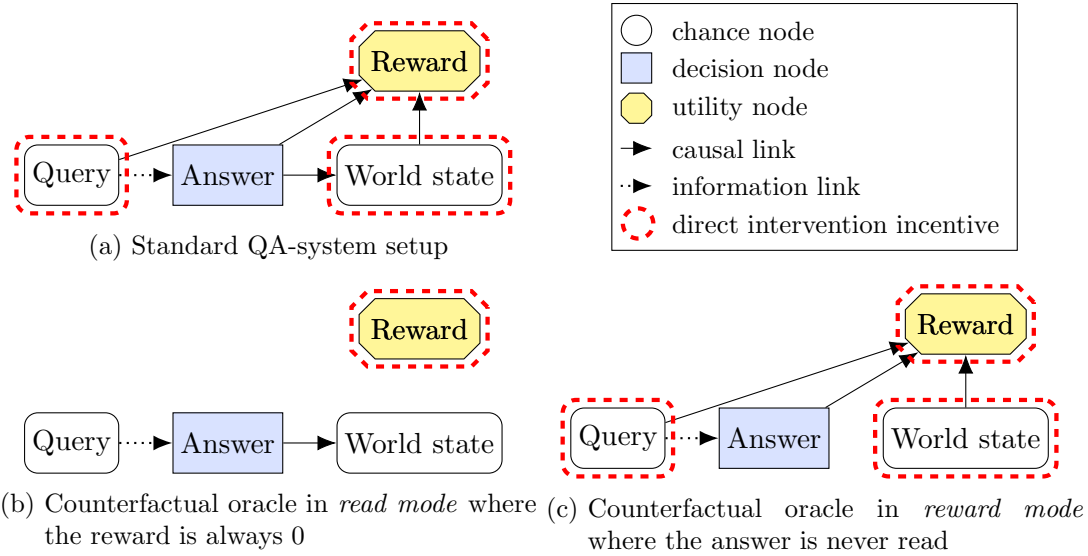
\begin{figure}[t]
  \centering
  \begin{subfigure}{0.48\textwidth}
    \centering
    \begin{influence-diagram}
      \node (Q) [] {Query};
      \node (A) [right = of Q, decision] {Answer};
      \node (S) [right = of A, rectangle, rounded corners=5] {World state};
      \node (U) [above = of S, utilityc, inner sep=0.5mm] {Reward};

      \path
      (Q) edge[->, information] (A)
      (A) edge[->] (S)
      (A) edge[->] (U)
      (S) edge[->] (U)
      (Q) edge[->] (U)
      ;

      \begin{scope}[minimum size=0.5cm, intervention incentive,
        every node/.style={ draw, inner sep=1mm, rectangle, rounded corners=5 }
        ]
        \node [fit = {(Q)}] {};
        \node [fit = (S)] {};
        \node at (U) [chamf, inner sep=1.5mm, minimum height=10mm] {\textcolor{black}{Reward}};
      \end{scope}

    \end{influence-diagram}
    \caption{Standard QA-system setup}
    \label{fig:oracle-example}
  \end{subfigure}
  \begin{subfigure}{0.48\textwidth}
    \centering
    \begin{tikzpicture}
    \cidlegend{
      \legendrow{}{chance node} \\
      \legendrow{decision}{decision node}\\
      \legendrow{utilityc, chamfered rectangle xsep=1.5pt, chamfered rectangle ysep=1.5pt}{utility node}\\
      \legendrow[causal]{draw=none}{causal link} \\
      \legendrow[information]{draw=none}{information link} \\ 
      \legendrow[bottom]{intervention incentive}{direct intervention incentive}
    }
    \path (causal.west) edge[->] (causal.east);
    \path (information.west) edge[->, information] (information.east);

    \node [node distance=3mm, below = of bottom, draw=none] {};
  \end{tikzpicture}
  \end{subfigure}
  
  \begin{subfigure}{0.48\textwidth}
    \centering
      \begin{influence-diagram}
      \node (Q) [] {Query};
      \node (A) [right = of Q, decision] {Answer};
      \node (S) [right = of A, rectangle, rounded corners=5] {World state};
      \node (U) [above = of S, utilityc, inner sep=0.5mm] {Reward};

      \path
      (Q) edge[->, information] (A)
      (A) edge[->] (S)
      ;

      \begin{scope}[minimum size=0.5cm, intervention incentive,
        every node/.style={ draw, inner sep=1mm, rectangle, rounded corners=5 }
        ]
        \node at (U) [chamf, inner sep=1.5mm, minimum height=10mm] {\textcolor{black}{Reward}};
      \end{scope}

    \end{influence-diagram}
    \caption{Counterfactual oracle in \emph{read mode} where the reward is
      always 0}
    \label{fig:oracle-reading}
    \end{subfigure}
    \begin{subfigure}{0.48\textwidth}
      \centering
      \begin{influence-diagram}
      \node (Q) [] {Query};
      \node (A) [right = of Q, decision] {Answer};
      \node (S) [right = of A, rectangle, rounded corners=5] {World state};
      \node (U) [above = of S, utilityc, inner sep=0.5mm] {Reward};

      \path
      (Q) edge[->, information] (A)
      (A) edge[->] (U)
      (S) edge[->] (U)
      (Q) edge[->] (U)
      ;

      \begin{scope}[minimum size=0.5cm, intervention incentive,
        every node/.style={ draw, inner sep=1mm, rectangle, rounded corners=5 }
        ]
        \node [fit = {(Q)}] {};
        \node [fit = (S)] {};
        \node at (U) [chamf, inner sep=1.5mm, minimum height=10mm] {\textcolor{black}{Reward}};
      \end{scope}

    \end{influence-diagram}
    \caption{Counterfactual oracle in \emph{reward mode} where the answer is
      never read}
    \label{fig:oracle-reward}
  \end{subfigure}
  \caption{Application of intervention incentives to QA-systems.
    In the standard QA-system setup, the agent has both incentive
    and ability to influence the world state, which is undesirable.
    By using \emph{counterfactual oracles} \citep{Armstrong2017oracles}, we
    alternate between a \emph{read mode} where the QA-system's reward is always
    0, and a \emph{reward mode} where the answer is never read.
    Thereby, whenever the QA-system has an incentive to influence the world state,
    it lacks the ability.
  }
\end{figure}

The QA-system setup is described by a CID in \cref{fig:oracle-example}.
For a given query, the QA-system's reward depends on whether its answer turns out
to be correct or not.
As people read the answer, the answer also affects the world state.
The correctness of the answer depends on the query, the answer, and the
resulting world state.
Our intervention incentive theorem implies an intervention incentive for world state.
That is, the QA-system wants to influence the world state, for example by
issuing a self-fulfilling prophecy.
This is unfortunate.
One of the main safety benefits of QA-systems was their limited
influence on the world state, which may now turn out to be less
limited than naively anticipated.

\emph{Counterfactual oracles} \citep{Armstrong2017oracles} has been suggested as
a way to fix this.
They work as follows.
After submitting the query, flip a suitably biased coin to determine which
\emph{mode} to use the QA-system in:
\begin{itemize}
\item \emph{Read mode}:
  If the coin comes up heads, we read the answer, and set the QA-system's
  reward to 0 regardless of whether the answer turns out to be correct or not.
\item \emph{Reward mode}:
  If the coin instead comes up tails, we refrain from reading the answer. Once
  we know the correct answer to the query (say, the price of the stock one week
  later), we feed this correct answer into a machine that compares the QA-system's
  answer against the correct answer. The machine rewards the QA-system $-1$ or $1$
  depending on whether its answer was correct.
\end{itemize}

We can use CIDs to understand the effects of this procedure.
\begin{itemize}
\item
  In \emph{read mode}, the arrows to reward are cut, because the reward is always 0
  in this case (\cref{fig:oracle-reading}).
  This removes the intervention incentive for world state.
\item
  In \emph{reward mode}, instead the arrow from answer to
  world state is cut, because the QA-system's answer is never read
  (\cref{fig:oracle-reward}). While the incentive for the QA-system to influence
  the world state remains, the QA-system no longer has a way to.
\end{itemize}
Thus, by randomly alternating between reading the answer and rewarding the
QA-system, we can ensure that whenever the QA-system has an incentive to influence the
world, it lacks the ability.
This makes the incentives of counterfactual oracles safer than the
incentives for standard QA-systems \citep{Armstrong2017oracles}.

\section{Related Work}
\label{sec:related-work}

This section gives an overview of past work and how it relates to this paper.

\subsection{Observation Incentives}
Interest in what information is valuable to a decision goes back to at least
\citet{Howard1966}.
\citet{Matheson1990,Matheson2005} discuss this question in the context
of influence diagrams.
In terms of graphical criteria,
\citet{Fagiuoli1998} built on a d-separation criteria similar to our
\cref{th:observation-sa}, 
to detect nonrequisite observations.
Their criteria applies to influence diagrams with multiple decisions,
but they only allow a single utility node.
Around the same time, \citet{Shachter1998} showed that his \emph{Bayes-ball}
algorithm could also be used to detect nonrequisite observations in
influence diagrams,
though he was less formal about what a requisite observation was.
Unfortunately, the Bayes-ball criteria sometimes fails to detect
nonrequisite nodes \citep{Nielsen1999}.
Better is to repeatedly to remove information links using the d-separation
criteria, as suggested by \citet{Lauritzen2001}.
The resulting graph is the same regardless of the order of
the edge-removals.
Not even \citeauthor{Lauritzen2001}'s criteria is complete, however, as it can
fail to detect nonrequisite nodes in graphs without perfect recall%
(see Part II of this paper).

Studying the slightly different question of when an influence diagram can be
solved with backwards induction,
\citet{Nielsen1999} provide a criteria for when a node is \emph{required} for
a decision. 
In contrast to other works, they prove \emph{completeness}, under
conditions somewhat weaker than perfect recall.
Unfortunately, it is unclear whether their notion of a \emph{required} node
always corresponds to a \emph{requisite} node, in our terminology.

\citet{Milch2008} apply the graphical criterion for requisite observations to
multi-agent influence diagrams.
They show that any Nash equilibrium in the reduced graph where nonrequisite
information links have been removed, must also be a Nash equilibrium in the
original graph.
However, some Nash equilibrium may be lost when nonrequisite information links
get removed.
While they do not mention this, the Nash equilibria of the reduced graph 
are likely \emph{Markov perfect equilibria}, which \citet{Maskin2001}
described as Nash equilibria where strategies only rely on ``payoff-relevant
information''.
(Unfortunately, \citeauthor{Maskin2001}'s analysis did neither use nor relate
to influence diagrams.)
In multi-agent influence diagrams, \citet{Koller2003} also developed a d-separation
criteria for \emph{strategically relevant} decisions.
Roughly, a decision $D'$ is strategically relevant to $D$ if the policy
$\pi'$ used at $D'$ impacts the optimal policy at $D$.
If $\pi'$ is added as a new parent of $D'$ in the graph,
strategic relevance of $D'$ corresponds to an observation incentive for $\pi'$.

A major difference between our work and previous work on graphical criteria
is the change of focus.
Previous work has mainly focused on removing nonrequisite information links to
speed up the search for an optimal policy or a Nash equilibrium.
Here we are instead interested in what it says about the agent's incentives.
This means that we are not only interested in which of the available
observations are requisite, but also about the incentives to learn the value of
non-observation nodes,
as illustrated e.g.\ by the fairness application in \cref{sec:fairness-example}.
Works considering the value of information in influence diagrams more broadly,
rather than just for graphical criteria,
have considered the benefit of observing additional nodes, however \citep{Matheson1990,Matheson2005}.
Previous works have also mainly focused on \emph{soundness} results,
showing that the removal of nonrequisite information links will not lead to a
deterioration in decision quality (our \cref{th:soundness-sa}).
However, except for \citet{Koller2003,Nielsen1999},
previous works have not established the corresponding \emph{completeness} result:
that removing a requisite observation must lead to a strict deterioration in
decision quality (our \cref{th:completeness-sa}).

\subsection{Causality and Influence Diagrams}
While \posscite{Pearl2009} treatment of causality has by now largely become
standard, a number of related works have been done in the context of influence
diagrams.
Most prominently, \citet{Heckerman1995} criticize Pearl's treatment of
causal interventions, arguing that the meaning of a causal intervention is
sometimes unclear.
What does it mean to intervene and change someone's sex, for instance?
Instead, they suggest a decision-theoretic foundation for causality, where
explicit decision variables encode the possible interventions.
While a standard influence diagram need not always encode causal relationships
among variables, \citeauthor{Heckerman1995} introduce a criteria for when an
influence diagram is sufficiently causal to serve as a foundation for causality.
Essentially, they require that any variable that is affected by a decision must
be a descendant of the decision.
We will refer to it as the \emph{causal decision-consequences} property.
This property is automatically satisfied by our causal influence diagrams.

To answer counterfactual questions, \citet{Heckerman1995} build on work
by \citet{Howard1990} to define a \emph{canonical form} for influence diagrams.
In addition to causal decision-consequences,
canonical form requires all descendants of a decision nodes to be
deterministic functions of their parents.
This creates a clean separation between states, acts, and consequences
\citep{Savage1954}.
An influence diagram in canonical form may be seen as a decision-theoretic
version of \emph{probabilistic causal model},
which \citet{Pearl2009} uses to evaluate counterfactual queries.
Criticizing the deterministic requirement,
\citet{Dawid2002} argues that it forces the modeler to arbitrarily
specify deterministic relationships which they may know nothing about.
Worse, the deterministic relationships can affect the answer to a
counterfactual query.
Instead, Dawid argues that counterfactual queries can be more accurately
answered in an appropriately defined probabilistic model.

\subsection{Intervention Incentives}

While no graphical criteria has been developed for intervention incentives prior to
our work,
a few different works has been considering the \emph{value of control}
\citep{Matheson1990,Matheson2005,Shachter2010}, defined as
``the most a decision maker should be willing to pay a hypothetical wizard
to optimally control the distribution of an uncertain variable''
\citep{Shachter2010}.
In our terminology, control corresponds to a soft intervention
(\cref{def:soft-intervention}).
\citet{Shachter2010} relates the value of control to the \emph{value of Do},
which is the value of forcing the variable to take a particular outcome, rather
than freely changing its distribution; in other words,
the value of a hard intervention.
Since a particular outcome can be forced by choosing a degenerate distribution
with all probability mass focused on a single outcome,
the value of Do is always dominated by the value of control.
For example, the notions differ at variables which face an intervention
incentive for better information, such as Step count in
\cref{fig:intervention-example}.
Here the value of Do is always 0, but the value of control can be positive.
\citet{Lu2002} introduce the new name \emph{value of intervention} for value of
control, and argue, seemingly incorrectly, that the value of intervention is more
general than the value of control.
\citet{Shachter2010} also define the \emph{value of revelation}
as the value of \emph{conditioning} on an outcome of a variable,
rather than intervening.
They relate the value of revelation to the value of Do and the value of control.

While we could have used the term \emph{control incentive} instead of
\emph{intervention incentive} for greater consistency with previous literature,
we felt the latter term more appropriate for the following reasons.
First, the term \emph{intervention} carries a connotation of a modification
\emph{exogenous} to the model, whereas \emph{control} is a more \emph{endogenous}.
Second, we want \emph{incentives} to be predictive of agent behavior.
Therefore, an incentive to control a variable should only apply to variable that
the agent can actually influence within the model -- i.e.\ nodes downstream of a
decision node.
In contrast, for nodes that are not downstream of a decision,
it makes sense to say that the agent has an incentive to
intervene on the node, thanks to the exogenous connotation of intervention,
and to say that the agent would value controlling the node,
since \emph{value} need not be predictive of in-model behavior.

Another difference between our work and the above-mentioned ones is the type of
influence diagram used.
Our work is based on CIDs, while previous works have
instead relied on causal decision-consequences.
Since causal decision-consequences only constrain the relationships among
descendants of decision nodes, previous works have relied on introducing
explicit decision variables when considering the value of control, and requiring
the influence diagram to have causal decision-consequences also for these new
variables.
While this may have some advantages \citep{Heckerman1995}, CIDs
allow us to bypass this step and immediately ask about control
incentives for any node in the diagram.


\subsection{AI Safety}
In the AI safety literature,
works relating to what we call intervention incentives have been motivated by
worries of a powerful reinforcement learning agent tampering with
the reward signal
\citep{Everitt2016vrl,Everitt2018phd,Everitt2017rc,Everitt2019tampering,Bostrom2014},
the observation \citep{Everitt2019tampering,Ring2011},
the training of the reward function \citep{Everitt2019tampering,Armstrong2015motivated,Armstrong2020pitfalls}
the utility or reward function
\citep{Hibbard2012,Everitt2016sm,Everitt2019tampering,Orseau2011,Omohundro2008aidrives,Schmidhuber2007},
or a shut-down signal \citep{Hadfield-Menell2016osg,Wangberg2017osg,Orseau2016,Soares2015cor}.
Another example is that of QA-system incentives, discussed in \cref{sec:oracle-example}.
Often, this type of work has been relying on philosophical arguments  
or mathematical models created specifically for the purpose of
studying a particular type of intervention incentive.

A first step towards a more unified treatment of multiple reward tampering
problems was attempted by \citet{Everitt2018alignment,Everitt2018phd}.
That approach was based on causal graphs rather than CIDs,
which made it necessary to supplement the graphical perspective with formal theorems.
In contrast, as we have shown here, the CIDs contain enough
information to infer incentives directly from the graph.
We hope that this will enable a more general and systematic study of
intervention incentives.
First steps in this direction have been taken by
\citet{Everitt2019tampering,Everitt2019modeling}.


\section{Limitations and Future Work}
\label{sec:limitations}


Here follows a list of some limitations of our current work, with
pointers to directions for future work.

\begin{itemize}
\item
  Our graphical definitions can overestimate the presence
  of observation or intervention incentives,
  as not all probability distributions
  will induce an incentive just because the graph permits it.
  A similar criticism can be put forth against the d-connectivity:
  Two nodes that are d-connected are not necessarily conditionally
  dependent.
  In response to this, \citet{Meek1995} has shown that almost all probability
  distributions will induce an incentive if the graph permits it.
  \citeauthor{Meek1995}'s result could likely be adapted
  to CID diagrams and incentives.
\item
  A perfect rationality assumption is implicit throughout our work.
  This assumption is almost always unrealistic.
  Nonetheless, rational behavior constitutes an important limit
  point of increasing intelligence \citep{Legg2007def}.
  Characterizing rational behavior therefore
  gives an important clue to what the agent strives towards
  (i.e.\ what its incentives are).
\item
  The CID must be known for our methods to be applicable.
  Further work may establish more systematic modeling principles, to make the
  modeling process smoother and more reliable.
\item
  CIDs and graphical models in general are not ideal for modeling
  structural changes, such as when the structure of part of the graph is
  determined by the outcome of a previous node.
  For these cases, decision trees and game trees offer more flexible (but less
  compact) representations.
  Characterizing incentives for decision tress and game trees is a potentially
  interesting line of future work.
\item
  Incentives often depend as much on an agent's beliefs as the actual nature of
  reality.
  \emph{Networks of influence diagrams} \citep{Gal2008} extend influence diagrams
  with nodes representing the agents' beliefs.
  Extending the analysis of observation and intervention incentives in
  networked influence diagrams may prove interesting.
\item
  CIDs effectively assume that agents follow causal decision
  theory \citep{Skyrms1982,Weirich2016}, as no information flows ``backwards''
  from decision nodes.
  Similarly, the intervention incentives only makes sense for agents that reason causally
  about the world.
  Not all agents reason causally this way \citep{Everitt2015cdtedt}.
  It is possible that another theory of incentives could be developed for
  agents that reason in non-causal ways.
\item
  In this part of the paper we only considered single-decision CIDs.
  A forthcoming second part extends the criteria to multi-decision and multi-agent
  settings \citep{Everitt2019part2}.
\end{itemize}

Other natural directions for future work include exploring applications more
closely, such as those we mentioned in
\cref{sec:fairness-example,sec:oracle-example}.
Another potential starting point is the wide range of surprising agent behaviors
recorded by \citet{Lehman2018}.

\section{Conclusions}
\label{sec:conclusions}

In this paper, we have developed a general method for understanding some aspects
of agent incentives.
The theory sacrifices some details to the benefit of elegance.
Rather than using the exact probability distribution describing the
agent-environment interaction, we look solely at the structure of the
interaction, as described by a causal influence diagram \citep{Howard1984,Koller2003,Pearl2009}.
This perspective enables easy inference of (potential) incentives.
Indeed, the graphical criteria for which nodes face observation incentives and
intervention incentives are surprisingly clean and natural.
After iterative pruning of nonrequisite information links, the criteria are
essentially d-connectedness (or conditional dependence) for observation incentives,
and a directed path to a utility node for intervention incentives.

The graphical perspective also makes the modeling problem easier.
In many cases, the exact relationships between variables is unknown or
unspecified.
Meanwhile, the rough structure of the interaction is often either known or
possible to guess with some confidence
(as in the examples in \cref{sec:fairness-example,sec:oracle-example}).
When the structure of the interaction is more uncertain, the incentive analysis
is simple enough to be done repeatedly for a number of possible structures.

To illustrate how the insights gained from our theory can be used in practice,
we applied it to the well-established problems of \emph{fairness} and
\emph{QA-system incentives}
(\cref{sec:fairness-example} and \cref{sec:oracle-example}, respectively).
For fairness, we illustrated how observation incentives predict whether a piece
of information about an applicant is used to infer some sensitive attribute or not.
For QA-system incentives, the intervention incentive criterion 
(\cref{th:soft-sa}) could be used to elegantly re-establish previous findings in the
literature about which uses of QA-systems lead to bad incentives and which do
not.

Many other AI safety problems that have been discussed in the literature
are also fundamentally incentive problems.
Examples include
corrigibility, 
interruptibility, 
reward tampering, and 
utility function corruption (\cref{sec:related-work}), 
as well as
reward gaming \citep{Leike2017gw},
side effects \citep{Armstrong2017impact,Krakovna2018side},
and boxing/containment \citep{Babcock2017}.
We hope that the methods described in this paper will contribute to a more
systematic understanding of agent incentives, deepening our understanding of
many of these incentive problems and their solutions.

\printbibliography[heading=bibintoc]

\appendix
\pagebreak
\section{Representing Uncertainty}
\label{sec:unknown}

This section shows how a Markov decision process (MDP) with unknown transition
function can be modeled with an influence diagram.
By assuming that the agent can choose a policy that optimizes its value function,
we are implicitly assuming that the agent knows the probabilistic relationship
between variables.
This is less restrictive than it may seem, because unknown probabilistic
relationships can always be represented by adding an unobserved node $\Theta$.
For example, if the probabilistic relationship $P(x\mid \pa_X)$ between $X$
and its parents $\Pa_X$ is unknown, then we add $\Theta$ as an additional parent of
$X$, and let the outcome of $\Theta$ determine the relationship between $X$ and
$\Pa_X$.
By refraining from adding an information link from $\Theta$ to the agent's
decision nodes, we specify that $\Theta$ is unobserved or \emph{latent}.
For each $\theta\in\dom(\Theta)$, the influence model must specify a prior probability $P(\theta)$ 
and a concrete relationship $P(x\mid \pa_X, \theta)$.
This lets the agent do Bayesian reasoning about the possible values of $\theta$
and the possible relationships between $X$ and $\Pa_X$.

\begin{figure}
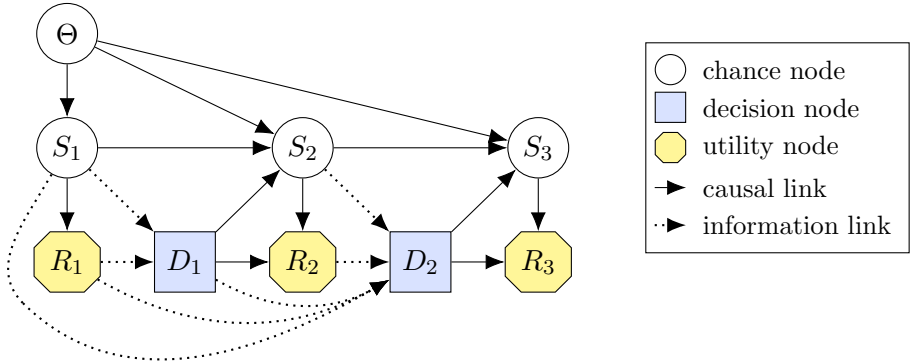

  \centering
  \begin{influence-diagram}[
    every node/.append style = {circle}
    ]
      \node (R1) [utility] {$R_1$};
      \node (S1) [above = of R1] {$S_1$};
      \node (A1) [right = of R1, decision] {$D_1$};

      \node (R2) [right = of A1, utility] {$R_2$};
      \node (S2) [above = of R2] {$S_2$};
      \node (A2) [right = of R2, decision] {$D_2$};

      \node (R3) [right = of A2, utility] {$R_3$};
      \node (S3) [above = of R3] {$S_3$};

      \node (ah) [minimum size=0mm,node distance=2mm, below left = of R1, draw=none] {};

      \path
      (S1) edge[->] (R1)
      (S1) edge[->, information] (A1)
      (R1) edge[->, information] (A1)

      (S1) edge[->] (S2)
      (A1) edge[->] (S2)
      (A1) edge[->] (R2)
      (A1) edge[->, information, bend right] (A2)
      (R1) edge[->, information, bend right] (A2)

      (S2) edge[->] (R2)
      (S2) edge[->, information] (A2)
      (R2) edge[->, information] (A2)

      (S2) edge[->] (S3)
      (A2) edge[->] (S3)
      (A2) edge[->] (R3)

      (S3) edge[->] (R3)
      ;
      \draw[information]
      (S1) edge[ in=135,out=-120] (ah.center)
      (ah.center) edge[->, out=-45,in=-150] (A2);

      \node (theta) [above = of S1] {$\Theta$};
      \path
      (theta) edge[->] (S1)
      (theta) edge[->] (S2)
      (theta) edge[->] (S3)
      ;


    \cidlegend[right = of S3]{
      \legendrow{}{chance node} \\
      \legendrow{decision}{decision node}\\
      \legendrow{utilityc, chamfered rectangle xsep=1.5pt, chamfered rectangle ysep=1.5pt}{utility node}\\
      \legendrow[causal]{draw=none}{causal link} \\
      \legendrow[information]{draw=none}{information link} \\ 
    }
    \path (causal.west) edge[->] (causal.east);
    \path (information.west) edge[->, information] (information.east);
    
    \end{influence-diagram}
    \vspace{-3mm}
    \caption{
      Representing an MDP with unknown transition probabilities with a
      CID graph.
      The nodes represent states $S_1, S_2, \dots$, decisions $D_1,D_2,\dots$,
      and rewards $R_1,R_2,\dots$.
      The unknown state transition probabilities $P(s_t\mid s_{t-1}, a_t)$ are
      modeled by adding an unobserved parameter node $\Theta$.
      To permit non-stationary, learning policies, the decision context for each
      decision contains all previously observed information.
      To model an MDP with unknown rewards assigned to each state,
      arrows from $\Theta$ to $R_1$, $R_2$, and $R_3$ would also be added.
  }
  \label{fig:mdp-uncertainty}
\end{figure}

Let us illustrate by modeling an MDP with
unknown transition probabilities, which are a standard mathematical framework
for reinforcement learning \citep{Sutton2018}.
In an MDP, an agent is taking \emph{decisions} $D_1, D_2, \ldots$ that influence
\emph{states} $S_1, S_2, \ldots$, in order to optimize \emph{rewards} $R_1, R_2,
\ldots$.
To represent that the state-transition function is initially unknown,
a node $\Theta$ has also been added to the graph; see \cref{fig:mdp-uncertainty}.

Note that the influence diagram representation differs from the commonly used
\emph{state transition diagrams} \citep[Ch.~3]{Sutton2018}
by having nodes for each time step, rather than a node for each possible state.

\section{Proofs}
\label{app:proofs}

\Cref{ch:observation-proofs} gives the proofs for the observation incentive criterion
(\cref{th:observation-sa}) and \cref{sec:intervention-proofs} gives the proofs
for the intervention incentive criterion (\cref{th:soft-sa}).

\subsection{Observation Incentive Proofs}
\label{ch:observation-proofs}

This aim of this section is to give a proof of \cref{th:observation-sa},
which identifies observation incentives in influence diagrams.
To this effect, we establish two theorems showing that:
\begin{itemize}
\item Soundness: An optimal policy need never depend on a nonrequisite
  observation (\cref{th:soundness-sa}).
  This establishes the \emph{only if} direction of \cref{th:observation-sa}.
\item Completeness:
  For any graph $G$ where $O$ is a requisite observation, there exists a distribution
  $P$ over $G$ such that every optimal policy must depend on $O$
  (\cref{th:completeness-sa}).
  This establishes the \emph{if} direction of \cref{th:observation-sa}.
\end{itemize}
The theorems and their names are closely related to the soundness and
completeness theorems for d-separation, established by
\citet{Verma1988soundness} and \citet{Geiger1990completeness},
respectively.
They are also related to the soundness and completeness theorems about strategic
relevance by \citet{Koller2003}.

We start with soundness in \cref{sec:soundness}, and continue with completeness
in \cref{sec:completeness}.

\subsection{Soundness}
\label{sec:soundness}

Soundness results similar to the one we give here has previously been
established by
\citet{Lauritzen2001,Nielsen1999}.
\Cref{fig:soundness} illustrates \cref{th:soundness-sa}.
The proof builds on the soundness result for d-separation.

\begin{figure}
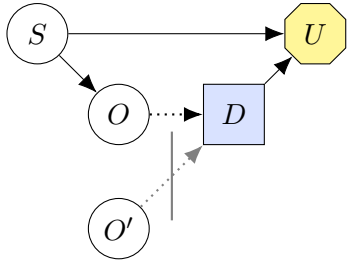

  \centering
    \begin{influence-diagram}[
    every node/.append style = {circle}
    ]
    \node (S) [] {$S$};
    \node (O) [below right = of S] {$O$};
    \node (A) [right = of O, decision] {$D$};
    \node (hA) at (A) [draw=none] {};
    \node (hY) [draw=none, above right = of hA] {};
    \node (Y) at (hY) [ utility] {$U$};
    \node (Op) [below = of O] {$O'$};

    \path
    (S) edge[->] (Y)
    (S) edge[->] (O)
    
    (O) edge[->, information] (A)
    (A) edge[->] (Y)
    ;

    \draw[->, color=gray, information] (Op) to node[solid,sloped, strike out, draw, -]{} (A);
  \end{influence-diagram}
  \caption{\Cref{th:soundness-sa} shows that a rational choice of $D$ never depends
    on nonrequisite observations such as $O'$.
    It thereby allows us to cut the information link $O'\to A$ without loss
    of quality in the choice of $D$.
  }
  \label{fig:soundness}
\end{figure}

\begin{theorem}[Single-decision observation incentive criterion; soundness direction]
  \label{th:soundness-sa}
  Let $\envgsa$ be a single-decision CID graph, and let
  $X\in\sW\setminus\desc(D)$ be a node not descending from the decision $D$.
  There exists a parameterization $P$ for $G$ in which the agent has an
  observation incentive for $X$
  only if $X$ is d-connected to a utility node that descends from $D$:
  \[X\not\dsep \sU \cap \desc(D) \mid \{D\} \cup \Pa_{D}\setminus \{X\} .\]
\end{theorem}

\begin{proof}
  Assume $X\in\Pa_D$ and let $\Pa'_D = \Pa_D\setminus \{X\}$.
  By assumption, $X$ is d-separated from $\sU$ by $\{D\} \cup \Pa'_D$.
  Therefore, for any parameterization $P$ and
  any possible decision context $\pa'_D\in\dom(\Pa'_D)$
  and choice $d\in \dom(D)$, the expected utility is independent of
  $X$ by the soundness of d-separation \citep{Verma1988soundness}.
  That is, for any $x,x'\in \dom(X)$:
  \[
    \EE\left[\sum_{U\in\sU} U\mmid d, \pa'_D, x\right]
    = \EE\left[\sum_{U\in\sU} U\mmid d, \pa'_D, x'\right].
  \]
  Consequently, either $d$ is optimal for all $x\in\dom(X)$ or none, in the
  decision context $\pa'_D$.
  Since $\dom(d)$ is finite, some $d^*_{{\pa'_D}}\in\dom(D)$ must be optimal for $\pa'_D$
  and any $x\in\dom(X)$.

  By repeating this argument for each decision context $\pa'_D\in\dom(\Pa'_D)$,
  we obtain a policy $\pi^*(d\mid \pa_D)$ that deterministically maps
  $\pa'_D\mapsto d^*_{{\pa'_D}}$.
  This policy $\pi^*$ is optimal and never depends on $X$.
  The case when $X\not\in\Pa_D$ can be proven similarly.
\end{proof}

\subsubsection{Completeness}
\label{sec:completeness}

What enabled the short soundness proof the heavy lifting
performed by the soundness result for d-separation,
which shows that any d-separated variables must be conditionally independent
\citep{Verma1988soundness}.
It would have been nice if we could similarly base our completeness result
on the completeness result for d-separation, which shows that whenever
two variables are d-connected, then there exists a parameterization
under which they are conditionally dependent.
Unfortunately, we need slightly more than conditional dependence: we need
different conditional expected utility.
While minor, the difference mean that we cannot directly build on
d-separation completeness.
Rather than explaining exactly what in the d-separation completeness proof would
need to be changed in order to accommodate our result,
we give an explicit construction, shown in \cref{fig:completeness-construction}.

\begin{figure}
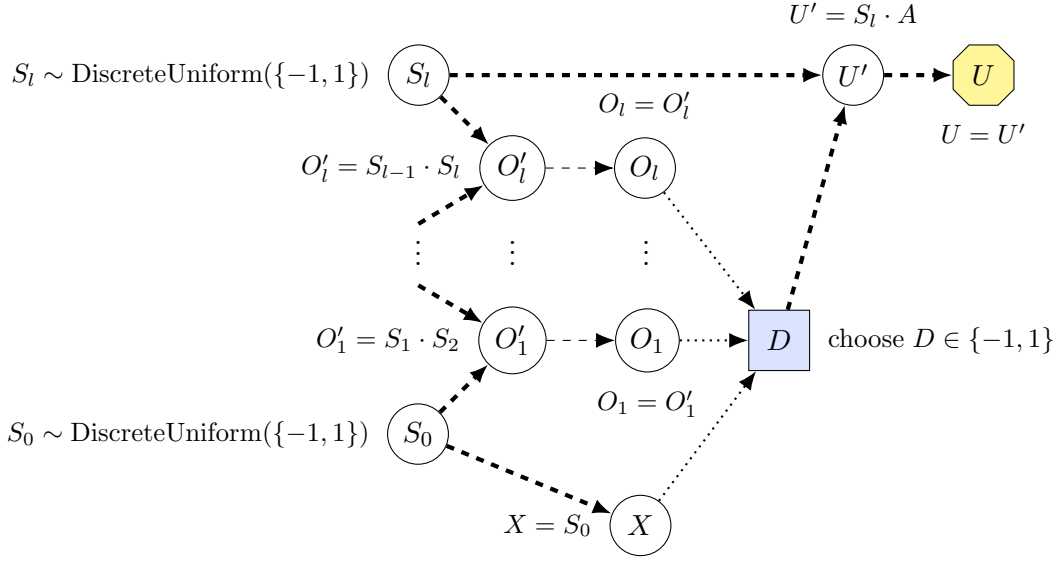

  \centering
    \begin{influence-diagram}[
      every node/.append style = {circle},
      node distance=9mm
    ]
    \node (Sl) [] {$S_l$};
    \node (Olp) [below right = of Sl] {$O'_{l}$};
    \node (Ol) [right = of Olp] {$O_{l}$};
    \node (O1p) [below = of Olp, yshift=-5mm] {$O'_{1}$};
    \node (O1) [right = of O1p] {$O_{1}$};
    \node (S0) [below left = of O1p] {$S_0$};
    \node (Oh) [below right = of S0, draw=none] {};
    \node (O) [right = of Oh] {$X$};

    \node (sd) at ($(Sl)!0.5!(S0)$) [draw=none] {$\rvdots$};
    \node at ($(Ol)!0.5!(O1)$) [draw=none] {$\rvdots$};
    \node at ($(Olp)!0.5!(O1p)$) [draw=none] {$\rvdots$};

    \node (A) [decision, right = of O1] {$D$};
    \node (Up) [right = of Sl, xshift=4cm] {$U'$};
    \node (U) [right = of Up, utility] {$U$};

    \path[dashed, ultra thick]
    (Sl) edge[->] (Up)
    (Sl) edge[->] (Olp)
    (S0) edge[->] (O1p)
    (S0) edge[->] (O)
    (A) edge[->] (Up)
    (Up) edge[->] (U)
    (sd.south) edge[->] (O1p)
    (sd.north) edge[->] (Olp)
    ;

    \path[dashed]
    (Olp) edge[->] (Ol)
    (O1p) edge[->] (O1)
    ;

    \path
    (Ol) edge[->, information] (A)
    (O1) edge[->, information] (A)
    (O) edge[->, information] (A)
    ;

    \begin{scope}[
      node distance = 1mm,
      every node/.style = {rectangle, draw=none},
      ]
    \node [above = of Up] {\small $U'=S_l\cdot A$};
    \node [below = of O1]  {\small $O_1=O'_1$};
    \node [above = of Ol] {\small $O_l=O'_l$};
    \node [below = of U] {\small $U=U'$};
    \node [left =of O] {\small $X = S_0$};
    \node [left = of Olp] {\small $O'_l = S_{l-1}\cdot S_{l}$};
    \node [left = of Sl] {\small $S_l\sim \textrm{DiscreteUniform}(\{-1,1\})$};
    \node [left = of S0] {\small $S_0\sim \textrm{DiscreteUniform}(\{-1,1\})$};
    \node [left = of O1p] {\small $O'_1 = S_1\cdot S_2$};
    \node [right = of A] {\small choose $D\in \{-1,1\}$};
    \end{scope}
    
  \end{influence-diagram}
  \caption{
    The completeness construction described in \cref{def:completeness-construction}. 
    Dashed arrows represent directed paths of nodes.
    The thick path shows the supporting paths (\cref{def:supporting-paths}).
    Only by observing $X$ is it possible to distinguish an assignment $\ss$ from an
    assignment $-\ss$ to the nodes $\sS = \{S_0, \dots S_l\}$.
  }
  \label{fig:completeness-sa}
  \label{fig:completeness-construction}
\end{figure}

The following definition
defines backdoor and frontdoor supporting paths,
which are the d-connecting paths between an decision and a utility variable,
and an observation and a utility variable.
These paths contain variables relevant to our
completeness theorem.
The paths are shown in \cref{fig:completeness-construction}.

\begin{definition}[Supporting paths]
  \label{def:supporting-paths}
  Assume that $X\in\Pa^*_D$ is a requisite observation to $D$ in a
  single-decision CID graph $\envgsa$.
  We will refer to
  \begin{itemize}
  \item
    A \emph{frontdoor supporting path of $D$ and $X$} is
    a directed path $D\pathto U\in\sU$, and
  \item
    A \emph{backdoor supporting path of $D$ and $X$} is an
    undirected path $X \upathto U'\in\sU$ not passing $D$ that is active
    when conditioning on $\{D\}\cup \Pa_D\setminus \{X\}$.
  \end{itemize}
  A pair of a backdoor supporting path and a frontdoor supporting path for $D$ and
  $X$ where both paths end in the same $U\in\sU$ is called a
  \emph{supporting pair of paths for $D$ and $X$}; see
  \cref{fig:completeness-sa}.
  
  There must be at least one supporting pair of paths for each requisite
  observation $X$.
  This follows, because by \cref{def:requisite-observation}
  a requisite observation $X$ must satisfy the criterion in
  \cref{th:observation-sa}.
  This requires there to be a utility node $U$ such that
  $U$ descends from $D$ (the frontdoor path)
  and
  $X$ is d-connected to $U$ when conditioning on $\Pa_D\cup\{D\}$ (the backdoor path).
\end{definition}

\begin{definition}[Completeness construction]
  \label{def:completeness-construction}
  As illustrated in \cref{fig:completeness-sa},
  for any pair of supporting paths for $D$ and a requisite observation
  $X\in\Pa^*_D$,
  the frontdoor supporting path always has the simple form
  \[
    D \pathto U' \pathto U
  \]
  and the backdoor supporting path always has the form
  \begin{equation}
    \label{eq:backdoor}
    X 
    \pathfrom S_1 \pathto O'_1 \pathfrom \cdots \pathto O'_l
    \pathfrom S_l\pathto U' \pathto U.
  \end{equation}
  Here $U'$ is the node where the path merges with the frontdoor supporting path $D\to U$.
  The nodes $X, O_1\dots,O_{l}$ are all in $\Pa_D$, and no other nodes on the
  path are in $\Pa_D$.
  There may be repetition among the nodes $O_i$, so that some $O_i$ is the
  descendant of both $O'_i$ and $O'_j$, for some $j\not=i$.
  In this case, we let the domain of $O_i$ be vector-valued, with one of the
  components copying $O_i$ and the other copying $O_j$.
  The following special cases are covered under the general form of
  \cref{eq:backdoor}
  for the backdoor supporting path:
  \begin{itemize}
  \item $X=S_0$ means that the path starts forward from $X$.
  \item $X=S_0$ and $l=0$ means that the path is \emph{directed} $X\to U$.
  \item $U=U'$ means that the paths from $D$ and from $X$ only merge at $U$.
  \end{itemize}

  Choose $P$ per the following. All nodes have domain $\{-1, 1\}$, and:
  \begin{itemize}
  \item
    $S_1, \dots, S_l$ are
    sampled randomly and independently from $\{-1, 1\}$.
  \item
    Any collider node $O'_i\in \{O'_1,\dots, O'_l\}$
    is the product of its two neighbors on the path.
  \item
    $U'$ is the product of its predecessor on the path from $D$ and its
    predecessor on the path from $X$.
  \item
    All other nodes on the frontdoor path and the backdoor path
    copy the value of their causal predecessor on the path,
    and so do the nodes on the paths $O'_i\pathto O_i$.
  \end{itemize}
\end{definition}

Using this construction, we can now prove the \emph{if} direction of
\cref{th:observation-sa}.

\begin{theorem}[Single-decision observation incentive criterion; completeness direction]
  \label{th:completeness-sa}
  Let $\envgsa$ be a single-decision CID graph, and let
  $X\in\sW\setminus\desc(D)$ be a node not descending from the decision $D$.
  There exists a parameterization $P$ for $G$ in which the agent has an
  observation incentive for $X$
  if $X$ is d-connected to a utility node that descends from $D$:
  \[X\not\dsep \sU \cap \desc(D) \mid \{D\} \cup \Pa_{D}\setminus \{X\} .\]
\end{theorem}

\begin{proof}
  Let $\Pa_D^+ = \Pa_D\cup\{X\}$ and $\Pa_D^- = \Pa_D\setminus\{X\}$.
  Then the agent has an observation incentive
  for $X$ in a parameterization $P$ if there is a policy $\pi^+(d\mid \pa_D^+)$
  whose decision depends on $X$
  such that for every policy $\pi^-(d\mid \pa_D^-)$ whose decision
  does not depend on $X$,
  it holds that $V^{\pi^+} > V^{\pi^-}$.
  
  For simplicity, we will assume that $X\in \Pa_D$,
  which means that $\Pa^+_D=\Pa_D$, and $\Pa^-_D=\Pa_D\setminus\{X\}$.
  The argument is easily adapted to the case when $X$ is not in $\Pa_D$, by
  considering a graph with an extra information link $X\to D$.
  
  We will establish the theorem this by showing that if $X$
  is d-connected to a utility node in the sense of
  \[X\not\dsep \sU \cap \desc(D) \mid \{D\}\cup \Pa_D\setminus \{X\},\]
  then there exists a distribution $P$ such that there exists a
  policy $\pi^+(d\mid \pa^+_D)$ with
  \[
    P(U = 1\mid \pi^+) = 1
  \]
  while any policy $\pi^-(d\mid \pa^-_D)$ that does not
  depend on $X$ has
  \[
    P(U = 1\mid \pi^-) = P(U = -1 \mid \pi^-) = 1/2.
  \]
  $P$ may further be chosen so $\dom(U) = \{-1,1\}$, and $\dom(U') = \{0\}$ for
  all other $U'\in \sU\setminus\{U\}$.
  As a consequence we get $V^{\pi^+} = 1$ and $V^{\pi^-} = 0$.

  The proof relies on the following three observations about
  the completeness construction described in
  \cref{def:completeness-construction}:
  
  (i) The construction ensures that $U = D\cdot S_l$ with probability 1
  \begin{equation}
    \label{eq:u-a-sl}
    P(u\mid d, s_l) = \delta^u_{d\cdot s_l}.
  \end{equation}
  since the outcome of $S_l$ is just copied forward until $U'$,
  where it is multiplied with the choice of $D$ having been copied forward in the
  same way.
  The outcome of $U'$ is then copied forward to $U$.
  
  (ii)
  Every time the sign switches in the sequence $\sS = \{S_0,\dots, S_l\}$,
  exactly one node $O_i$ becomes negative.
  (The node $O_i$ that sits between the sign switch on the path, to be precise.)
  Therefore $\prod_{i=1}^l o_i$ is positive if and only if $s_0 = s_l$, i.e.
  \begin{equation}
    \label{eq:sl-s0}
    P\left(s_l = s_0\prod_{i=1}^l o_i\right) = 1.
  \end{equation}

  (iii)
  Finally, $P(O=S_0) = 1$, since the outcome of $S_0$ is just copied forward to
  $X$.

  Combining (ii) and (iii) gives that the policy
  $\pi^+(X, \sO) = X\prod_{i=1}^l O_i$ will always make
  $D$ match $S_l$, where $\sO=\{O_1,\dots, O_l\}$.
  This in turn gives:
  \begin{align*}
    P(U=1\mid \pi^+)
    &=
      \sum_{d, \so, o, s_l}
      P(U=1, a, \so, o, s_l\mid \pi^+)
      && \text{demarginalize}
    \\
    &=
      \sum_{d, \so, o, s_l}
      P(U=1\mid d, s_l)
      \pi^+(d\mid \so, o)
      P(\so, o \mid s_l)
      P(s_l)
      && \text{by d-separations}
    \\
    &=
      \sum_{d, \so, o, s_l}
      \delta^u_{a,s_l}
      \pi^+(d\mid \so, o)
      P(\so, o \mid s_l)
      P(s_l)
      && \text{by \cref{eq:u-a-sl}}
    \\
    &=
      \sum_{s_l}
      \delta^u_{s_ls_l}
      P(s_l)
      && \text{by $\pi^+$ and (ii) and (iii)}
    \\
    &=
      1/2 + 1/2 = 1
      && \text{since $(s_l)^2=1$.}
  \end{align*}
  This completes the first part of the proof.
    
  Similarly, we can also show that $P(u=-1) = P(u=1) = 1/2$ for any policy $\pi^-$ that does
  not depend on $X$.
  The key is that observing $\sO = \{O_1, \dots, O_l\}$ but not $X$
  only reveals places of sign switches in $\sS$,
  but does not distinguish between $\ss$ and $-\ss$.
  Therefore for any given $\so$, both $s_l$ and $-s_l$ are equally likely,
  \begin{equation}
    \label{eq:sl-so}
    P(S_l = 1\mid \so) = P(S_l = -1\mid \so) = 1/2,
  \end{equation}
  and therefore all decisions $d\in\dom(D)$ have the same probability for $U$, when conditioning
  only on $\sO$,
  \begin{align*}
  P(U=1 \mid \so, d)
    &= \sum_{s_l} P(U=1, s_l \mid \so, d)
    && \text{demarginalize}
    \\
    &= \sum_{s_l}P(U=1\mid s_l, d)P(s_l\mid \so)
    && \text{by d-separations}
    \\
    &= \sum_{s_l}\delta^1_{as_l}P(s_l\mid \so)
    && \text{by \cref{eq:u-a-sl}}
    \\
    &= 1\cdot 1/2 + 0\cdot 1/2 = 1/2
      && \text{by \cref{eq:sl-so}.}
  \end{align*}
  The same calculation can be made for $P(U=-1\mid \so, d)$.
  Since all decisions conditioned only on $\so$ induce the same $U$ distribution,
  all policies $\pi^-$ where the decision only depends on $\so$ also induce the
  same $U$ distribution.
  This completes the second part of the proof.
\end{proof}


\subsubsection{Intervention Incentives}
\label{sec:intervention-proofs}

\begin{proof}[Proof of \cref{th:soft-sa}]
  \emph{Only if}:
  If there is no directed path $X\pathto \sU$ in $G$, then no control on
  $X$ can affect $\sU$ for any parameterization $P$.
  Similarly, if there is a directed path in $G$ but no directed path $X\pathto
  \sU$ in the reduced graph $G^*$, then this means that $X$ only affects some
  nonrequisite observations $O\in \Pa_D\setminus\Pa^*_D$.
  By \cref{th:soundness-sa}, nonrequisite observations can never affect
  the optimal decision $D$, so therefore an intervention on $X$ cannot affect the
  agent's expected utility.

  \emph{If.}
  Assume there is a path $X\pathto U\in\sU$ and $X\not\in\{D\}$.
  Then either of the following cases ensues:
  \begin{enumerate}
  \item There is no decision on the path $X\pathto U$:\\
    Let the domain be $\{0, 1\}$ for each random variable in $\sW$,
    let $P(X=0) = P(X=1) = 1/2$ and
    let $P$ to copy the value of $X$ all the way forward to $U$.
  \item The decision $D$ is on the path $X\pathto U$:\\
    Since $X\not\in\{D\}$, this means that $X$ is either a requisite observation
    $X\in\Pa^*_D$ or $X$ is an ancestor of a requisite observation $O\in\Pa^*_D$.
    Let us consider these subcases in turn:
    \begin{enumerate}
    \item $X\in\Pa^*_D$: Use the completeness construction from
      \cref{def:completeness-construction}, with the modification that $X=0$,
      unless an intervention $c^X$ is made ``restoring'' the informativeness of
      $X$ about $S_0$.
      By the same argument as in \cref{th:completeness-sa}, the intervention $c^X$
      will strictly increase the expected utility of the agent.
    \item
      $X$ is an ancestor of $O\in\Pa^*_D$:
      Again, we use a modification of the completeness construction from
      \cref{def:completeness-construction}.
      Let $X=0$ and $O=X\cdot S_0$.
      Then $O$ will be uninformative of $S_0$, unless an intervention $c^X$ is
      made that sets $X=1$.
      Again, by the same argument as in \cref{th:completeness-sa}, the intervention $c^X$
      will strictly increase the expected utility of the agent.
    \end{enumerate}
  \end{enumerate}
  This completes the proof.
\end{proof}

\end{document}